\documentclass{article}
\usepackage{microtype}
\usepackage{graphicx}
\usepackage{subfigure}
\usepackage{nicefrac}
\usepackage{booktabs} 

\usepackage{hyperref}

\usepackage[accepted]{icml2021}

\icmltitlerunning{Achieving Instance-Optimality and Minimax-Optimality in Stochastic and Adversarial Linear Bandits}
\usepackage{afterpage}
\usepackage{prettyref}
\usepackage{framed}
\newcommand{\pref}[1]{\prettyref{#1}}

\newcommand{\savehyperref}[2]{\texorpdfstring{\hyperref[#1]{#2}}{#2}}
\newrefformat{eq}{\savehyperref{#1}{Eq. \textup{(\ref*{#1})}}}
\newrefformat{eqn}{\savehyperref{#1}{Eq.~\textup{(\ref*{#1})}}}
\newrefformat{lem}{\savehyperref{#1}{Lemma~\ref*{#1}}}
\newrefformat{lemma}{\savehyperref{#1}{Lemma~\ref*{#1}}}
\newrefformat{def}{\savehyperref{#1}{Definition~\ref*{#1}}}
\newrefformat{line}{\savehyperref{#1}{Line~\ref*{#1}}}
\newrefformat{thm}{\savehyperref{#1}{Theorem~\ref*{#1}}}
\newrefformat{corr}{\savehyperref{#1}{Corollary~\ref*{#1}}}
\newrefformat{cor}{\savehyperref{#1}{Corollary~\ref*{#1}}}
\newrefformat{sec}{\savehyperref{#1}{Section~\ref*{#1}}}
\newrefformat{app}{\savehyperref{#1}{Appendix~\ref*{#1}}}
\newrefformat{assum}{\savehyperref{#1}{Assumption~\ref*{#1}}}
\newrefformat{ex}{\savehyperref{#1}{Example~\ref*{#1}}}
\newrefformat{fig}{\savehyperref{#1}{Figure~\ref*{#1}}}
\newrefformat{alg}{\savehyperref{#1}{Algorithm~\ref*{#1}}}
\newrefformat{rem}{\savehyperref{#1}{Remark~\ref*{#1}}}
\newrefformat{conj}{\savehyperref{#1}{Conjecture~\ref*{#1}}}
\newrefformat{prop}{\savehyperref{#1}{Proposition~\ref*{#1}}}
\newrefformat{proto}{\savehyperref{#1}{Protocol~\ref*{#1}}}
\newrefformat{prob}{\savehyperref{#1}{Problem~\ref*{#1}}}
\newrefformat{claim}{\savehyperref{#1}{Claim~\ref*{#1}}}
\newrefformat{que}{\savehyperref{#1}{Question~\ref*{#1}}}
\newrefformat{op}{\savehyperref{#1}{Open Problem~\ref*{#1}}}
\newrefformat{fn}{\savehyperref{#1}{Footnote~\ref*{#1}}}

\usepackage{xspace}
\usepackage{amsmath}
\usepackage{amsthm}
\usepackage{bbm}
\usepackage{algorithm}
\usepackage[algo2e, vlined, ruled]{algorithm2e}
\usepackage{algorithmic}
\usepackage{mathrsfs}
\usepackage{amssymb}
\usepackage{bm}
\usepackage{makecell}
\usepackage{tabulary}

\definecolor{Green}{rgb}{0.13, 0.65, 0.3}

\DeclareMathOperator*{\argmin}{argmin} 

\newcommand{\Reg}{\text{\rm Reg}}

\newcommand{\ghp}{\textsc{GeometricHedge.P}\xspace}

\newcommand{\calA}{\mathcal{A}}
\newcommand{\calB}{\mathcal{B}}

\newcommand{\calS}{\mathcal{S}}

\newcommand{\calI}{\mathcal{I}}

\newcommand{\E}{\mathbb{E}}
\newcommand{\order}{\mathcal{O}}

\newcommand{\hatell}{\widehat{\ell}}

\newcommand{\Otil}{\widetilde{\order}}

\newcommand{\calP}{\mathcal{P}}

\newcommand{\calX}{\mathcal{X}}

\newcommand{\one}{\mathbbm{1}}
\newcommand{\inner}[1]{\langle#1\rangle}

\newcommand{\smallgap}{\vspace*{3pt}}
\newcommand{\OP}{\textbf{OP}\xspace}
\newcommand{\clip}{\text{Clip}}
\newcommand{\delt}{\kappa}

\newcommand{\tildeS}{\widetilde{S}}
\newcommand{\tildep}{\widetilde{p}}
\newcommand{\Catoni}{\textbf{Catoni}}
\newcommand{\Rob}{\text{Rob}}
\newcommand{\ellhat}{\widehat{\ell}}
\DeclareMathOperator{\Tr}{Tr}

\newcommand{\alg}{\textsc{BOTW}\xspace}
\newcommand{\algse}{\textsc{BOTW-SE}\xspace}
\newcommand{\term}[1]{\textsc{Term #1}\xspace}

\newcommand{\hatDelta}{\widehat{\Delta}}

\newcommand{\hatx}{\widehat{x}}

\newcommand{\dev}{\textsc{Dev}}

\newcommand{\logdt}{\lg}

\newtheorem*{theorem*}{Theorem}
\newtheorem{theorem}{Theorem}

\newtheorem{assumption}{Assumption}
\newtheorem{lemma}[theorem]{Lemma}

\newtheorem*{lemma*}{Lemma}
\newtheorem{definition}{Definition}

\theoremstyle{definition}

\definecolor{green}{rgb}{0.13, 0.8, 0.3}
\definecolor{purple}{rgb}{0.8, 0, 0.8}
\definecolor{lightblue}{rgb}{0.3, 0.6, 1.0}
\usepackage[]{color-edits}

\addauthor{XZ}{red}
\addauthor{CW}{lightblue}
\addauthor{MZ}{green}
\addauthor{CL}{purple} 
\addauthor{HL}{red}

\newcommand{\fT}{f_T}
\newcommand{\hmm}{\gamma_m}  
\newcommand{\htt}{\beta_t}
\newcommand{\htT}{\beta_T}
\newcommand{\gmm}{\gamma}
\newcommand{\gtt}{\beta}
\newcommand{\optconst}{2^{15}}

\begin{document}

\twocolumn[
\icmltitle{Achieving Near Instance-Optimality and Minimax-Optimality \\ in Stochastic and Adversarial Linear Bandits Simultaneously}



\icmlsetsymbol{equal}{*}

\begin{icmlauthorlist}
\icmlauthor{Chung-Wei Lee}{*}
\icmlauthor{Haipeng Luo}{*}
\icmlauthor{Chen-Yu Wei}{*}
\icmlauthor{Mengxiao Zhang}{*}
\icmlauthor{Xiaojin Zhang}{ii}
\end{icmlauthorlist}

\icmlaffiliation{*}{University of Southern California}
\icmlaffiliation{ii}{The Chinese University of Hong Kong}
\icmlcorrespondingauthor{Mengxiao Zhang}{mengxiao.zhang@usc.edu}
\icmlkeywords{Machine Learning, ICML}

\vskip 0.3in
]



\printAffiliationsAndNotice{* Authors are listed in alphabetical order.} 

\begin{abstract}
In this work, we develop linear bandit algorithms that automatically adapt to different environments.
By plugging a novel loss estimator into the optimization problem that characterizes the instance-optimal strategy, our first algorithm not only achieves nearly instance-optimal regret in stochastic environments, but also works in corrupted environments with additional regret being the amount of corruption,
while the state-of-the-art~\citep{li2019stochastic} achieves neither instance-optimality nor the optimal dependence on the corruption amount.
Moreover, by equipping this algorithm with an adversarial component and carefully-designed testings, our second algorithm {\it additionally} enjoys minimax-optimal regret in completely adversarial environments, which is the first of this kind to our knowledge.
Finally, all our guarantees hold with high probability, while existing instance-optimal guarantees only hold in expectation. 
\end{abstract}


\section{Introduction}
We consider the linear bandit problem with a finite and fixed action set.
In this problem, the learner repeatedly selects an action from the action set and observes her loss whose mean is the inner product between the chosen action and an unknown loss vector determined by the environment. 
The goal is to minimize the \emph{regret}, which is the difference between the learner's total loss and the total loss of the best action in hindsight.
Two standard environments are heavily-studied in the literature: the stochastic environment and the adversarial environment.
In the stochastic environment, the loss vector is fixed over time,
and we are interested in instance-optimal regret bounds of order $o(T^\epsilon)$ for any $\epsilon>0$, where $T$ is the number of rounds and $o(\cdot)$ hides some instance-dependent constants. 
On the other hand, in the adversarial environment, the loss vector can be arbitrary in each round, and we are interested in minimax-optimal regret bound  $\Otil(\sqrt{T})$, where $\Otil(\cdot)$ hides the problem dimension and logarithmic factors in $T$.

While there are many algorithms obtaining such optimal bounds in either environment (e.g., \citep{lattimore2017end} in the stochastic setting and \citep{bubeck2012towards} in the adversarial setting), a natural question is whether there exists an algorithm achieving both guarantees simultaneously without knowing the type of the environment.
Indeed, the same question has been studied extensively in recent years for the special case of multi-armed bandits where the action set is the standard basis~\citep{bubeck2012best,seldin2014one,auer2016algorithm,seldin2017improved,wei2018more,zimmert2019optimal}.
Notably, \citet{zimmert2019optimal} developed an algorithm that is optimal up to universal constants for both stochastic and adversarial environments,
and the techniques have been extended to combinatorial semi-bandits \citep{zimmert2019beating} and finite-horizon tabular Markov decision processes \citep{jin2020simultaneously}.  
Despite all these advances, however, it is still open whether similar results can be achieved for general linear bandits.

On the other hand, another line of recent works study the robustness of stochastic linear bandit algorithms from a different perspective and consider a corrupted setting where 
an adversary can corrupt the stochastic losses up to some limited amount $C$.
This was first considered in multi-armed bandits \citep{lykouris2018stochastic,gupta2019better,zimmert2019optimal,zimmert2021tsallis} and later extended to linear bandits \citep{li2019stochastic} and Markov decision processes \citep{lykouris2019corruption, jin2020simultaneously}.
Ideally, the regret of a robust stochastic algorithm should degrade with an additive term $\order(C)$ in this setting,
which is indeed the case in~\citep{gupta2019better,zimmert2019optimal, zimmert2021tsallis, jin2020simultaneously} for multi-armed bandits or Markov decision processes, but is not achieved yet for general linear bandits.


In this paper, we make significant progress in this direction and develop
algorithms with near-optimal regret simultaneously for different environments.
Our main contributions are as follows.
\begin{itemize}
    \item In \pref{sec: warmup}, we first introduce \pref{alg:REOLB}, a simple algorithm that achieves $\order(c(\calX, \theta)\log^2 T+C)$ regret with high probability in the corrupted setting,\footnote{In the texts, $\order(\cdot)$ often hides lower-order terms (in terms of $T$ dependence) for simplicity. However, in all formal theorem/lemma statements, we use $\order(\cdot)$ to hide universal constants only.} where $c(\calX,\theta)$ is an instance-dependent quantity such that the instance-optimal bound for the stochastic setting (i.e. $C=0$) is $\Theta(c(\calX, \theta)\log T)$. This result significantly improves~\citep{li2019stochastic} which only achieves
    $\order\big( \frac{d^6\log^2 T}{\Delta_{\min}^2} + \frac{d^{2.5}C\log T}{\Delta_{\min}}\big)$
    where $d$ is the dimension of the actions and $\Delta_{\min}$ is the minimum sub-optimality gap satisfying
    $c(\calX, \theta)\leq \order\big(\frac{d}{\Delta_{\min}}\big)$.
    Moreover, \pref{alg:REOLB} also ensures an instance-independent bound $\Otil(d\sqrt{T}+C)$ that some existing instance-optimal algorithms fail to achieve even when $C=0$ (e.g.,~\citep{jun2020crush}).
    
    \item In \pref{sec: alg}, based on \pref{alg:REOLB}, we further propose \pref{alg: BOBW} which not only achieves nearly instance-optimal regret $\order(c(\calX, \theta)\log^2 T)$ in the stochastic setting, but also achieves the minimax optimal regret 
  $\Otil(\sqrt{T})$ in the adversarial setting (both with high probability).
    To the best of our knowledge, this is the first algorithm that enjoys the best of both worlds for linear bandits.
    Additionally, the same algorithm also guarantees $\Otil\big(\frac{d\log^2T}{\Delta_{\min}}+C\big)$ in the corrupted setting, which is slightly worse than  \pref{alg:REOLB} but still significantly better than~\citep{li2019stochastic}. 
    
    \item Finally, noticing the extra $\log T$ factor in our bound for the stochastic setting, in \pref{app:lower_bound} we also prove that this is in fact inevitable if the same algorithm simultaneously achieves sublinear regret in the adversarial setting with high probability (which is the case for \pref{alg: BOBW}). This generalizes the result of \citep{auer2016algorithm} for two-armed bandits.  
    
\end{itemize}

At a high level, \pref{alg:REOLB} utilizes a well-known optimization problem (that characterizes the lower bound in the stochastic setting) along with a robust estimator to determine a randomized strategy for each round.
This ensures the near instance-optimality of the algorithm in the stochastic setting, and also the robustness to corruption when combined with a doubling trick. 
To handle the adversarial setting as well, \pref{alg: BOBW} switches between an adversarial linear bandit algorithm with high-probability regret guarantees and a variant of \pref{alg:REOLB}, depending on the results of some carefully-designed statistical tests on the stochasticity of the environment.


\section{Related Work}
\paragraph{Linear Bandits.}
Linear bandits is a classic model to study sequential decision problems. 
The stochastic setting dates back to~\citep{abe1999associative}.
\citet{auer2002using} first used the optimism principle to solve this problem.
Later, several algorithms were proposed based on confidence ellipsoids, further improving the regret bounds \citep{dani2008stochastic,rusmevichientong2010linearly,abbasi2011improved,chu2011contextual}.  

On the other hand, the adversarial setting was introduced by~\citet{awerbuch2004adaptive}.
\citet{dani2008price} achieved the first $\order(\sqrt{T})$ expected regret bound using the Geometric Hedge algorithm (also called Exp2) with uniform exploration over a barycentric spanner. \citet{abernethy2008competing} proposed the first computational efficient algorithm that achieves $\widetilde{\order}(\sqrt{T})$ regret using the Following-the-Regularized-Leader framework. 
\citet{bubeck2012towards} further tightened the bound by improving Exp2 with John's exploration.
Our \pref{alg: BOBW} makes use of any adversarial linear bandit algorithm with high-probability guarantees (e.g.,~\citep{bartlett2008high,lee2020bias}) in a black-box manner.

\paragraph{Instance Optimality for Bandit Problems.}
In the stochastic setting,
\citet{lattimore2017end} showed that, unlike multi-armed bandits, 
optimism-based algorithms or Thompson sampling can be arbitrarily far from optimal in some simple instances.
They proposed an algorithm also based on the lower bound optimization problem to achieve instance-optimality,
but their algorithm is deterministic and cannot be robust to an adversary.
Instance-optimality was also considered in other related problems lately such as linear contextual bandits \citep{hao2020adaptive,Tirinzoni2020}, partial monitoring \citep{komiyama2015regret}, and structured bandits \citep{combes2017minimal,jun2020crush}. 
Most of these works only consider expected regret, while our guarantees all hold with high probability. 

\paragraph{Best-of-Both-Worlds.}
Algorithms that are optimal for both stochastic and adversarial settings were studied in multi-armed bandits \citep{bubeck2012best,seldin2014one,auer2016algorithm,seldin2017improved,wei2018more,zimmert2019optimal}, semi-bandits \citep{zimmert2019beating}, and Markov Decision Processes \citep{jin2020simultaneously}.
On the other hand, linear bandits, a generalization of multi-armed bandits and semi-bandits, is much more challenging and currently underexplored in this direction. 
To the best of our knowledge, our algorithm is the first that guarantees near-optimal regret bounds in both stochastic and adversarial settings simultaneously.

\paragraph{Stochastic Bandits with Corruption.}
\citet{lykouris2018stochastic} first considered the corrupted setting for multi-armed bandits. 
Their results were improved by \citep{gupta2019better, zimmert2019optimal, zimmert2021tsallis} and extended to linear bandits \citep{li2019stochastic,bogunovic2020stochastic} and reinforcement learning \citep{lykouris2019corruption}.
As mentioned, our results significantly improve those of~\citep{li2019stochastic} (although their corruption model is slightly more general than ours; see~\pref{sec:prelim}).
On the other hand, the results of \citep{bogunovic2020stochastic} are incomparable to ours, because they consider a setting where the adversary has even more power and can decide the corruption after seeing the chosen action.
Finally, we note that \citep[Theorem 3.2]{lykouris2019corruption} considers episodic linear Markov decision processes in the corrupted setting, which can be seen as a generalization of linear bandits. However, this result is highly suboptimal when specified to linear bandits ($\Omega(C^2\sqrt{T})$ ignoring other parameters).


\section{Preliminaries}\label{sec:prelim}

Let $\calX\subset \mathbb{R}^d$ be a finite set that spans $\mathbb{R}^d$.
Each element in $\calX$ is called an \emph{arm} or an \emph{action}.
We assume that $\|x\|_2\le 1$ for all $x\in\calX$.
A linear bandit problem proceeds in $T$ rounds.
In each round $t=1, \ldots, T$, the learner selects an action $x_t \in \calX$.
Simultaneously, the environment decides a hidden loss vector $\ell_t\in\mathbb{R}^d$
and generates some independent zero-mean noise $\epsilon_t(x)$ for each action $x$.
Afterwards, the learner observes her loss $y_t =  \inner{x_t, \ell_t}+\epsilon_t(x_t)$.
We consider three different types of settings: stochastic, corrupted, and adversarial,
explained in detail below.

In the stochastic setting, $\ell_t$ is fixed to some unknown vector $\theta\in\mathbb{R}^d$.
We assume that there exists a unique optimal arm $x^*\in\calX$ such that $\inner{x^*,\theta}< \min_{x^* \neq x\in\calX} \inner{x,\theta}$,
and define for each $x\in\calX$,  its \emph{sub-optimality gap} as $\Delta_x = \inner{x-x^*,\theta}$.
Also denote the minimum gap $\min_{x\neq x^*}\Delta_x$ by $\Delta_{\min}$. 

The corrupted setting is a generalization of the stochastic setting, where in addition to a fixed vector $\theta$, the environment also decides a corruption vector $c_t \in \mathbb{R}^d$ for each round (before seeing $x_t$) so that $\ell_t = \theta + c_t$.\footnote{%
In other words, the environment corrupts the observation $y_t$ by adding $\inner{x_t, c_t}$. The setting of~\citep{li2019stochastic} is slightly more general with the corruption on $y_t$ being $c_t(x_t)$ for some function $c_t$ that is not necessarily linear.\label{fn:corruption_model}} 
We define the total amount of corruption as $C=\sum_t\max_{x\in\calX}|\inner{x, c_t}|$.
The stochastic setting is clearly a special case with $C=0$.
In both of these settings, we define the regret as
$
\Reg(T)=\max_{x\in\calX}\sum_{t=1}^T\inner{x_t-x,\theta} = \sum_{t=1}^T \Delta_{x_t}
$.

Finally, in the adversarial setting, $\ell_t$ can be chosen arbitrarily (possibly dependent on the learner's algorithm and her previously chosen actions). 
The difference compared to the corrupted setting (which also has potentially arbitrary loss vectors) is that the regret is now defined in terms of $\ell_t$: $\Reg(T)=\max_{x\in\calX}\sum_{t=1}^T\inner{x_t-x,\ell_t}$.

In all settings, we assume $\inner{x,\theta}, \inner{x,c_t}, \inner{x, \ell_t}$ and $y_t$ are all in $[-1,1]$ for all $t$ and $x\in\calX$.  We also denote $\inner{x,\ell_t}$ by $\ell_{t,x}$ and similarly $\inner{x,c_t}$ by $c_{t,x}$.




It is known that the minimax optimal regret in the adversarial setting is $\Theta(d\sqrt{T})$~\citep{dani2008price,bubeck2012towards}.
The instance-optimality in the stochastic case, on the other hand, is slightly more complicated.
Specifically, an algorithm is called \emph{consistent} if it guarantees $\E[\Reg(T)]=o(T^\epsilon)$ for any $\theta$, $\calX$, and $\epsilon>0$.
Then, a classic lower bound result (see e.g.,~\citep{lattimore2017end}) states that:
for a particular instance $(\calX, \theta)$, all consistent algorithms satisfy:\footnote{The original proof is under the Gaussian noise assumption. To meet our boundedness assumption on $y_t$, it suffices to consider the case when $y_t$ is a Bernoulli random variable, which only affects the constant of the lower bound.}
\begin{align*}
    \liminf_{T\to\infty}\frac{\E[\Reg(T)]}{\log T}\ge \Omega(c(\calX,\theta)),
\end{align*}
where $c(\calX,\theta)$ is the objective value of the following optimization problem:
\begin{align}
    &\inf_{N\in[0, \infty)^\calX} \sum_{x\in \calX\backslash\{x^*\}}N_x\Delta_x \label{eqn: opt-c-obj-main}\\
    \text{subject to\ \ } & \|x\|^2_{H^{-1}(N)}\leq \frac{\Delta_x^2}{2},\quad\forall {x\in \calX\backslash\{x^*\}} \label{eqn: opt-c-con-main}
\end{align}
and $H(N)=\sum_{x\in\calX}N_x xx^\top$ (the notation $\|x\|_{M}$ denotes the quadratic norm $\sqrt{x^\top M x}$ with respect to a matrix $M$).  
This implies that the best instance-dependent bound for $\Reg(T)$ one can hope for is $\order(c(\calX, \theta)\log T)$
(and more generally $\order(c(\calX, \theta)\log T + C)$ for the corrupted setting).
It can be shown that $c(\calX, \theta)\leq \order\left(\frac{d}{\Delta_{\min}}\right)$ (see \pref{lem: opt-constant}), but this upper bound can be arbitrarily loose as shown in~\citep{lattimore2017end}.

The solution $N_x$ in the optimization problem above specifies the least number of times action $x$ should be drawn in order to distinguish between the present environment and any other alternative environment with a different optimal action. Many previous instance-optimal algorithms try to match their number of pulls for $x$ to the solution $N_x$ under some estimated gap $\hatDelta_x$~\citep{lattimore2017end, hao2020adaptive, jun2020crush}. While these algorithms are asymptotically optimal, their regret usually grows linearly when $T$ is small \citep{jun2020crush}. Furthermore, they are all deterministic algorithms and by design cannot tolerate corruptions. We will show how these issues can be addressed in the next section. 

\paragraph{Notations. } We use $\calP_{\calS}$ to denote the probability simplex over $\calS$: $\left\{p\in\mathbb{R}_{\geq 0}^{|\calS|}: \sum_{s\in\calS} p_s=1\right\}$, and define the clipping operator $\clip_{[a,b]}(v)$ as $\min(\max(v,a), b)$ for $a\leq b$. 


\section{A New Algorithm for the Corrupted Setting}\label{sec: warmup}

In this section, we focus on the corrupted setting (hence covering the stochastic setting as well). We introduce a new algorithm that achieves with high probability an instance-dependent regret bound of $\order(c(\calX, \theta)\log^2 T+C)$ for large $T$ and also an instance-independent regret bound of $\Otil(d\sqrt{T}+C)$ for any $T$. 
This improves over previous instance-optimal algorithms \citep{lattimore2017end, hao2020adaptive, jun2020crush} from several aspects:
1) first and foremost, our algorithm handles corruption optimally with extra $\order(C)$ regret, while previous algorithms can fail completely due to their deterministic nature;
2) previous bounds only hold in expectation;
3) previous algorithms might suffer linear regret when $T$ is small, while ours is always $\Otil(d\sqrt{T}+C)$ for any $T$.
The price we pay is an additional $\log T$ factor in the instance-dependent bound. 
On the other hand, compared to the work of \citep{li2019stochastic} that also covers the same corrupted setting and achieves $\order\left( \frac{d^6\log^2 T}{\Delta_{\min}^2} + \frac{d^{2.5}C\log T}{\Delta_{\min}} \right)$, our results are also significantly better (recall $c(\calX, \theta)\leq \order(d/\Delta_{\min})$),
although as mentioned in \pref{fn:corruption_model}, their results hold for an even more general setting with non-linear corruption.

\afterpage{
\begin{algorithm}
    \nl \textbf{Input}: $\delta<0.1$ \\
    \nl $t\leftarrow 1$. \\
    \nl \For{$m=0,1,2\ldots$}{
        \nl Define block $\calB_m = \{t,\  t+1, \ldots,\  t+2^m-1\}$.
        \\
        \nl Find a randomized strategy $p_m=\OP(2^m, \hatDelta_m)$   \label{line: solving OP in alg 1} with \\
        $\hatDelta_{m,x}= \begin{cases}
        0 &\text{if $m=0$,} \\
        \Rob_{m-1,x}-\min_{x'\in \mathcal{X}}\Rob_{m-1,x'} &\text{else.}
        \end{cases}
        $ \\
        \nl   Compute second moment $S_m = \sum_{x\in \calX}p_{m,x}xx^\top$. \\
        \nl \While{$t\in \calB_m$}{
            \nl Sample $x_t\sim p_m$ and observe $y_t$. \\
            \nl Compute for all $x\in\calX$, $\ellhat_{t,x}=x^\top S_m^{-1}x_ty_t$.  \label{line: construct loss estimator alg 1}\\
            \nl $t\leftarrow t+1$.
        }
        \nl \For{$x\in \mathcal{X}$}{
            \nl Construct robust loss estimators \label{line: construct robust alg 1} 
            \[
            \Rob_{m,x}=\clip_{[-1,1]} \left(\Catoni_{\alpha_x}\left(\big\{\ellhat_{\tau,x}\big\}_{\tau\in\calB_m}\right)\right)
            \]
            with 
           $\alpha_x=\sqrt{\frac{4\log(2^m|\calX|/\delta)}{2^{m}\cdot\|x\|^2_{S_{m}^{-1}}+2^m}}$. 
        }
    }
    \caption{Randomized Instance-optimal Algorithm}\label{alg:REOLB}
\end{algorithm}
\begin{figure}[h!]
\begin{framed}
\textbf{OP}$(t,\hatDelta)$: 
return any minimizer $p^*$ of the following:
\begin{align}
    &\min_{p\in \calP_{\mathcal{X}}}\sum_xp_x \hatDelta_{x},\label{eqn: opt-2-objective}\\ 
    \text{s.t.\ \ } 
    &\|x\|_{S(p)^{-1}}^2 \le \frac{t\hatDelta_{x}^2}{\htt}+4d, \ \ \forall {x\in \calX},\label{eqn: opt-2-constraint}
\end{align}
where $ S(p)=\sum_{x\in \calX}p_{x}xx^\top$ and $\htt=\optconst\log\frac{t|\calX|}{\delta}$. \\
\end{framed}
\caption{Optimization Problem (OP)}
\label{fig: op}
\end{figure}
\begin{figure}[h!]
\begin{framed}
$\textbf{Catoni}_{\alpha}\left(\{X_1, X_2, \ldots, X_n\}\right):$ 
return $\widehat{X}$, the unique root of the function $f(z) = \sum_{i=1}^n \psi(\alpha(X_i-z))$
where
$
\psi(y) = \begin{cases}
\ln(1+y+y^2/2), &\text{if $y\geq0$,} \\
-\ln(1-y+y^2/2), &\text{else.}
\end{cases}
$
\end{framed}
\caption{Catoni's Estimator}\label{fig:catoni}
\end{figure}
}

Our algorithm is presented in \pref{alg:REOLB}, which proceeds in blocks of rounds whose length grows in a doubling manner ($2^0, 2^1, \ldots$). At the beginning of block $m$ (denoted as $\calB_m$), we compute a distribution $p_m$ over actions by solving an optimization problem \OP (\pref{fig: op}) using the empirical gap $\hatDelta_{m,x}$ estimated in the previous block (\pref{line: solving OP in alg 1}). Then we use $p_m$ to sample actions for the entire block $m$, and construct an unbiased loss estimator $\hatell_{t,x}$ in every round for every action $x$ (\pref{line: construct loss estimator alg 1}). At the end of each block $m$, we use $\{\ellhat_{\tau,x}\big\}_{\tau\in\calB_m}$ to construct a \emph{robust loss estimator} $\Rob_{m,x}$ for each action (\pref{line: construct robust alg 1}), which will then be used to construct $\hatDelta_{m+1, x}$ for the next block. 
We next explain the optimization problem \OP and the estimators in detail.

\OP is inspired by the lower bound optimization (\pref{eqn: opt-c-obj-main} and \pref{eqn: opt-c-con-main}), where we normalize the pull counts $N$ as a distribution $p$ over the arms such that for a large $m$, $p_{m,x}\approx \frac{N_x}{2^m}$ holds for $x \neq x^*$.
One key difference between our algorithm and previous ones~\citep{lattimore2017end, hao2020adaptive, jun2020crush} is exactly that we select actions randomly according to these distributions, while they try to deterministically match the pull count of each arm to $N$.
Our randomized strategy not only prevents the environment from exploiting the knowledge on the learner's choices, but also allows us to construct unbiased estimator $\ellhat_{t,x}$ (\pref{line: construct loss estimator alg 1}) following standard adversarial linear bandit algorithms~\citep{dani2008price, bubeck2012towards}.
In fact, as shown in our analysis, the variance of the estimator $\ellhat_{t,x}$ is exactly bounded by $\|x\|_{S_m^{-1}}^2$ (for $t\in\calB_m$), which is in turn bounded in terms of the sub-optimality gap of $x$ in light of the constraint \pref{eqn: opt-2-constraint}.
The similar idea of imposing explicit constraints on the variance of loss estimators appears before in for example~\citep{dudik2011efficient, agarwal2014taming} for contextual bandits.
Finally, we point out that $\OP$ always has a solution due to the additive term $4d$ in \pref{eqn: opt-2-constraint} (see \pref{lem: exist-d-distribution}), and it can be solved efficiently by standard methods since \pref{eqn: opt-2-constraint} is a convex constraint.

Another important ingredient of our algorithm is the robust estimator $\Rob_{m,x}$, 
which is a clipped version of the Catoni's esimator~\citep{catoni2012challenging} constructed using all the unbiased estimators $\{\hatell_{\tau,x}\}_{\tau\in\calB_m}$ from this block for action $x$ (\pref{fig:catoni}).
From a technical perspective, this avoids a lower-order term in Bernstein-style concentration bounds and is critical for our analysis.
We in fact also believe that this is necessary since there is no explicit regularization on the magnitude of $\hatell_{t,x}$, and it can indeed have a heavy-tailed distribution.
While other robust estimators are possible, we use the Catoni's estimator which was analyzed in \citep{wei2020taking} for non-i.i.d. random variables (again important for our analysis).

The following theorem summarizes the nearly instance-optimal regret bound of \pref{alg:REOLB}.
\begin{theorem}\label{thm: sto-alg-guarantee}
In the corrupted setting, \pref{alg:REOLB} guarantees that with probability at least $1-\delta$,
\begin{align*}
    \Reg(T) &= \order\Bigg(c(\calX, \theta)\log T\log\frac{T|\calX|}{\delta}+M^*\log^\frac{3}{2}\frac{1}{\delta} \\
    &\qquad \qquad  + C + d\sqrt{\frac{C}{\Delta_{\min}}} \log\frac{C|\calX|}{\Delta_{\min}\delta} \Bigg), 
\end{align*}
where 
$M^*$ is some constant that depends on $\calX$ and $\theta$ only. 
\end{theorem}

The dominating term of this regret bound is thus $\order(c(\calX, \theta)\log^2 T+C)$ as claimed.
The definition of $M^*$ can be found in the proof (\pref{app: sto}) and is importantly independent of $T$.
In fact, in \pref{thm: no-M-dependency}, we also provide an alternative (albeit weaker) bound $\order(\frac{d^2(\log T)^2}{\Delta_{\min}}+C)$ for \pref{alg:REOLB} without the dependence on $M^*$.


The next theorem shows an instance-independent bound of order $\Otil(d\sqrt{T}+C)$ for \pref{alg:REOLB}, which previous instance-optimal algorithms fail to achieve as mentioned.
\begin{theorem}\label{thm: regret_part1-main}
     In the corrupted setting, \pref{alg:REOLB} guarantees that with probability at least $1-\delta$, $\Reg(T)\leq \order(d\sqrt{T}\log(T|\calX|/\delta)+C)$.
\end{theorem}

We emphasize that \pref{alg:REOLB} is parameter-free and does not need to know $C$ to achieve these bounds.
In the rest of the section, we provide a proof sketch for \pref{thm: sto-alg-guarantee} and \pref{thm: regret_part1-main}. First, we show that the estimated gap $\hatDelta_{m,x}$ is close to the true gap $\Delta_x$ with a constant multiplicative factor and some additive terms that go down at the rate of roughly $1/\sqrt{t}$ up to the some average amount of corruption.

\begin{lemma}\label{lem:twice-main}
With probability at least $1-\delta$,  \pref{alg:REOLB} ensures for all $m$ and all $x$,
    \begin{align}
        \Delta_{x} &\leq 2\hatDelta_{m,x} + \sqrt{\frac{d \hmm}{4\cdot 2^m}} + 2\rho_{m-1},  \label{eqn: hyp1-double} \\
        \hatDelta_{m,x} &\leq 2\Delta_{x} + \sqrt{\frac{d \hmm}{4\cdot 2^m}} + 2\rho_{m-1}, \label{eqn: hyp2-double}
    \end{align}
    where $\rho_{m}=\sum_{k=0}^m\frac{2^kC_k}{4^{m-1}}$ ($\rho_{-1}$ is defined as $0$), $C_k=\sum_{\tau\in \calB_k}\max_{x\in \calX}|c_{\tau, x}|$ is the amount of corruption within block $k$, and $\hmm=\optconst\log(2^m|\calX|/\delta)$.
\end{lemma}

As mentioned, the proof of \pref{lem:twice-main} heavily relies on the robust estimators we use as well as the variance constraint \pref{eqn: opt-2-constraint}.
Next, we have the following lemma which bounds the objective value of \OP.

\begin{lemma}\label{lem: optimization-property-main}
    Let $p$ be the solution of $\OP(t,\hatDelta)$, where $\hatDelta\in \mathbb{R}_{\geq 0}^{|\calX|}$. Then we have $\sum_{x\in \calX}p_x\hatDelta_x= \order\left(\frac{d\log(t|\calX|/\delta)}{\sqrt{t}}\right)$.
\end{lemma}

Combining \pref{lem:twice-main} and \pref{lem: optimization-property-main}, we see that in block $m$, the regret of \pref{alg:REOLB} can be upper bounded by 
\begin{align*}
    &\order\left(2^m\sum_x  p_{m,x}\Delta_x \right)\\
    &= \Otil\left(2^m\sum_x  p_{m,x}  \left(\hatDelta_{m,x} + \sqrt{\frac{d}{2^{m}}} + \rho_{m-1}\right) \right) \\
    &= \Otil\left(d\sqrt{2^m} + 2^m \rho_{m-1} \right), 
\end{align*}
where in the first equality we use \pref{lem:twice-main} and in the second equality we use \pref{lem: optimization-property-main} with the fact that $p_m=\OP(2^m,\hatDelta_{m})$. Further summing this over $m$ and relating $\sum_{m}2^m \rho_{m-1}$ to $C$ proves \pref{thm: regret_part1-main}.

In addition, based on \pref{lem:twice-main}, we show that when $t \in \calB_m$ is larger than $\Omega(C/\Delta_{\min})$ plus some problem-dependent constant,
the estimated gap $\hatDelta_{m,x}$ becomes $\Theta(\Delta_x)$. Therefore, the solution $\{p_{m,x}\}_{x\in \calX\setminus\{x^*\}}$ from \OP is very close to $\{\frac{N_x}{2^m}\}_{x\in \calX\setminus\{x^*\}}$, where $N_x$ is the optimal solution of \pref{eqn: opt-c-obj-main} and \pref{eqn: opt-c-con-main}, except that we have an additional $\log(2^m|\calX|/\delta)$ factor in the constraint (coming from $\beta_{2^m}$). Therefore, the regret is bounded by $\order(c(\calX, \theta)\log(T)\log(T|\calX|/\delta))$ for large enough $T$. Formally, we have the following lemma.

\begin{lemma}\label{lem: regret_part2-main}
    \pref{alg:REOLB} guarantees with probability at least $1-\delta$, for some constant $T^*$ depending on $\calX, \theta$, and $C$:
    \begin{align*}
        \sum^{T}_{t=T^*+1}\sum_xp_{t,x}\Delta_x\le \order\left( c(\calX, \theta) \log(T)\log(T|\calX|/\delta)\right).
    \end{align*}
\end{lemma}
Finally, to obtain \pref{thm: sto-alg-guarantee}, it suffices to apply \pref{thm: regret_part1-main} for the regret before round $T^*$ and \pref{lem: regret_part2-main} for the regret after. 
 
\section{Best of Three Worlds}\label{sec: alg}

In this section, building on top of \pref{alg:REOLB}, we develop another algorithm that enjoys similar regret guarantees in the stochastic or corrupted setting,
and additionally guarantees $\Otil(\sqrt{T})$ regret in the adversarial setting,
without having any prior knowledge on which environment it is facing. To the best of our knowledge, this kind of \emph{best-of-three-worlds} guarantee only appears before
for multi-armed bandits~\citep{wei2018more, zimmert2019optimal} and Markov decision processes~\citep{jin2020simultaneously}, but not for linear bandits.  

Our algorithm requires a block-box access to an adversarial linear bandit algorithm $\calA$ that satisfies the following: 
\begin{assumption}\label{assum:  adversarial alg} 
     $\calA$ is a linear bandit algorithm that outputs a loss estimator $\hatell_{t,x}$ for each action $x$ after each time $t$. 
     There exist $L_0$, $C_1\geq \optconst d\log(T|\calX|/\delta)$, and universal constant $C_2\geq 20$, such that
     for all $t\geq L_0$, $\calA$ guarantees the following with probability at least $1-\frac{\delta}{T}$: $\forall x\in\calX$,
     \begin{align}
         \sum_{s=1}^t (\ell_{s,x_s} -  \ell_{s,x}) \leq \sqrt{C_1t} - C_2\left\vert \sum_{s=1}^t (\ell_{s,x} - \hatell_{s,x}) \right\vert.   \label{eqn: eqL guarantee 1} 
     \end{align}
\end{assumption} 

\pref{eqn: eqL guarantee 1} states that the regret of $\calA$ against action $x$ is bounded by a $\sqrt{t}$-order term minus the deviation between the loss of $x$ and its estimator.
While this might not seem intuitive, in fact, \emph{all} existing linear bandit algorithms with a near-optimal high-probability bound satisfy \pref{assum: adversarial alg},
even though this may not have been stated explicitly (and one may need to slightly change the constant parameters in these algorithms to satisfy the conditions on $C_1$ and $C_2$). 
Below, we give two examples of such $\calA$ and justify them in \pref{app: example algorithm}. 
\begin{itemize}
    \item A variant of GeometricHedge.P \citep{bartlett2008high} with an improved exploration scheme satisfies \pref{assum:  adversarial alg} with ($\delta'=\delta/(|\calX|\log_2 T)$) \[C_1=\Theta\left(d\,
    {\log(T/\delta')}
    \right), \ \ L_0=\Theta\left(d\,
    {\log^2(T/\delta')}
    \right).\] 
    \item The algorithm of \citep{lee2020bias} satisfies \pref{assum:  adversarial alg} with ($\logdt=\log(dT)$, $\delta''=\delta/(|\calX|T)$) \[C_1=\Theta\left(d^6\logdt^8{\log(\logdt/\delta'')}\right), \ \ L_0=\Theta\left(\log(\logdt/\delta'')\right).
    \] 
\end{itemize}
 
With such a black-box at hand, our algorithm \alg is shown in \pref{alg: BOBW}.
We first present its formal guarantees in different settings. 
\begin{theorem}\label{thm: sto-and-corrupt}
         \pref{alg: BOBW} guarantees that with probability at least $1-\delta$, in the stochastic setting ($C=0$), $\Reg(T)$ is at most
        \begin{align*}
          \order\left(c(\calX,\theta)\log T\log\frac{T|\calX|}{\delta}+\frac{C_1\sqrt{\log T}}{\Delta_{\min}}\right.\\
          \left.+M^*\log^{\frac{3}{2}}\frac{1}{\delta} + \sqrt{C_1L_0}\right),
        \end{align*}
        where $M^*$ is the same problem-dependent constant as in \pref{thm: sto-alg-guarantee}; and in the corrupted setting ($C>0$), $\Reg(T)$ is at most
        \begin{align*} \order\left(\frac{C_1\log T}{\Delta_{\min}}+C+\sqrt{C_1L_0}\right).
        \end{align*}
        In the case when $\calA$ is the variant of GeometricHedge.P, the last bound is
        \begin{align*} \order\left(\frac{d\log(T|\calX|/\delta)\log T}{\Delta_{\min}}+C\right).
        \end{align*}
\end{theorem}

Therefore, \pref{alg: BOBW} enjoys the nearly instance-optimal regret $\order(c(\calX, \theta)\log^2 T)$ in the stochastic setting as \pref{alg:REOLB}\footnote{Note that when we choose $\calA$ as the variant of GeometricHedge.P, $\frac{C_1\sqrt{\log T}}{\Delta_{\min}}=\order\left(\frac{d}{\Delta_{\min}}\log^{\frac{3}{2}}T\right)$ which is dominated by the term $\order(c(\calX, \theta)\log^2 T)$ when $T$ is sufficiently large.}, but slightly worse regret $\order(\frac{d\log^2 T}{\Delta_{\min}}+C)$ in the corrupted setting (recall again $c(\calX, \theta) \leq d/\Delta_{\min}$).
In exchange, however, \pref{alg: BOBW} enjoys the following worst-case robustness in the adversarial setting. 
\begin{theorem}\label{thm: fully adversarial}
    In the adversarial setting, \pref{alg: BOBW} guarantees that with probability at least $1-\delta$, 
    $\Reg(T)$ is at most $\order\left(\sqrt{C_1 T\log T}+\sqrt{C_1L_0}\right)$. 
\end{theorem}

The dependence on $T$ in this bound is minimax-optimal as mentioned, while the dependence on $d$ depends on the coefficient $C_1$ of the black-box. Note that because of this adversarial robustness, the $\log^2 T$ dependence in \pref{thm: sto-and-corrupt} turns out to be unavoidable, as we show in \pref{thm: lower bound}. In addition, \pref{thm: fully adversarial} also works for the stochastic setting, which implies a regret bound of $\order(\sqrt{dT}\log(T|\calX|\log_2 T/\delta))$. This is a factor of $\sqrt{d}$ better than the guarantee of \pref{alg:REOLB} shown in \pref{thm: regret_part1-main}.

Next, in \pref{sec: overview main algorithm}, we describe our algorithm in detail. Then in \pref{sec: adv} and \pref{sec: sto}, we provide proof sketches for \pref{thm: fully adversarial} and \pref{thm: sto-and-corrupt} respectively. 

\begin{algorithm}[t]
     \caption{\alg (Best of Three Worlds)}\label{alg: BOBW}
     \textbf{Input}: an algorithm $\calA$ satisfying \pref{assum:  adversarial alg}. \\
     \textbf{Initialize}: $L\leftarrow L_0$ ($L_0$ defined in \pref{assum:  adversarial alg}). \\
     \While{\text{true}}{
         Run \algse with input $L$, and receive output $t_0$. \\
         $L\leftarrow 2t_0$.
     }
\end{algorithm}

\subsection{The algorithm}
\label{sec: overview main algorithm}

\pref{alg: BOBW} BOTW takes a black-box $\calA$ satisfying \pref{assum:  adversarial alg} (with parameter $L_0$) as input, and then proceeds in epochs until the game ends.
In each epoch, it runs its single-epoch version \algse (\pref{alg: subroutine}) with a minimum duration $L$ (initialized as $L_0$).
Based on the results of some statistical tests, at some point \algse will terminate with an output $t_0 \geq L$.
Then BOTW enters into the next epoch with $L$ updated to $2t_0$, so that the number of epochs is always $\order(\log T)$.

\algse has two phases. In Phase 1, the learner executes the adversarial linear bandit algorithm $\calA$. Starting from $t=L$ (i.e. after the minimum duration specified by the input), the algorithm checks in every round whether \pref{eqn:  alg 7 jump condition 1} and \pref{eqn:  alg 7 jump condition 2} hold for some action $\hatx$ (\pref{line: check 1}). If there exists such an $\hatx$, Phase 1 terminates and the algorithm proceeds to Phase 2. 
This test is to detect whether the environment is likely stochastic. 
Indeed, \pref{eqn:  alg 7 jump condition 1} and \pref{eqn:  alg 7 jump condition 2} 
imply that the performance of the learner is significantly better than all but one action (i.e., $\hatx$). In the stochastic environment, this event happens at roughly $t\approx \Theta\left(\frac{d}{\Delta_{\min}^2}\right)$ with $\hatx=x^*$. This is exactly the timing when the learner 
should stop using $\calA$ whose regret grows as $\Otil(\sqrt{t})$ and start doing more exploitation on the better actions, in order to keep the regret logarithmic in time for the stochastic environment. 
We define $t_0$ to be the time when Phase 1 ends, and $\hatDelta_{x}$ be the empirical gap for action $x$ with respect to the estimators obtained from $\calA$ so far (\pref{line: hat_delta_x}). In the stochastic setting, we can show that  $\hatDelta_x=\Theta(\Delta_x)$ holds with high probability.  

In the second phase, we calculate the action distribution using \OP with the estimated gap $\{\hatDelta_x\}_{x\in\calX}$. Indeed, if $\hatDelta_x$'s are accurate, the distribution returned by \OP is close to the optimal way of allocating arm pulls, leading to near-optimal regret.\footnote{Here, we solve \OP at every iteration for simplicity. It can in fact be done only when time doubles, just like \pref{alg:REOLB}.}
For technical reasons, there are some differences between Phase 2 and \pref{alg:REOLB}. First, instead of using $p_t$, the distribution returned by \OP, to draw actions, we mix it with $e_{\hatx}$ (the distribution that concentrates on $\hatx$), and draw actions using $\tildep_t=\frac{1}{2}e_{\hatx} + \frac{1}{2}p_t$. This way, $\hatx$ is drawn with probability at least $\frac{1}{2}$. Moreover, the loss estimator $\hatell_{t,x}$ is now defined as the following: 
\begin{align}
    \hatell_{t,x} = \begin{cases}
        x^\top \tildeS_t^{-1}x_ty_t, \qquad &x\neq \hatx\\
        \frac{y_t}{\tildep_{t,\hatx}}\mathbb{I}\{x_t=\hatx\}, \qquad &x=\hatx 
    \end{cases}   \label{eqn: hatell construction}
\end{align}
where $\tildeS_t = \sum_{x\in\calX} \tildep_{t,x}xx^\top$. While the construction of $\hatell_{t,x}$ for $x\neq \hatx$ is the same as \pref{alg:REOLB}, we see that the construction of $\hatell_{t,\hatx}$ is different and is based on standard inverse probability weighting. These differences are mainly because we later use the average estimator instead of the robust mean estimator for $\hatx$ (the latter produces a slightly looser concentration bound in our analysis). Therefore, we must ensure that $\hatx$ is drawn with enough probability, and that the magnitude of $\hatell_{t,\hatx}$ is well-controlled.

\setcounter{AlgoLine}{0}
\begin{algorithm}[t]
     \caption{\algse (BOTW -- Single Epoch)}
     \label{alg: subroutine}
     \textbf{Input}: $L$ (minimum duration) \\
     \textbf{Define}: $f_T = \log T$  \\
     \textbf{Initialize}: a new instance of $\calA$. \\
     \texttt{// Phase 1} \\
     \nl \For{$t=1, 2,\ldots$}{\label{line: phase1_start}
         \nl Execute and update $\calA$. Receive estimators $\{\hatell_{t,x}\}_{x\in\calX}$. \\
         \nl \textbf{if} $t\geq L$ and there exists an action $\hatx$ such that\label{line: alg_jump_condition}
         \begin{align}
             \sum_{s=1}^t y_s - \sum_{s=1}^t \hatell_{s,\hatx}  &\geq  -5\sqrt{f_TC_1t}, \label{eqn:  alg 7 jump condition 1}  \\
             \sum_{s=1}^t y_s - \sum_{s=1}^t \hatell_{s,x}  &\leq -25 \sqrt{f_TC_1 t}, \ \ \  \forall x\neq \hatx,  \label{eqn: alg 7 jump condition 2}  
         \end{align} 
         \label{line: check 1}
         \nl \textbf{then}\ $t_0 \leftarrow t, \; \hatDelta_x\leftarrow \frac{1}{t_0}\left(\sum_{s=1}^{t_0}\hatell_{s,x} - \hatell_{s,\hatx}\right)$, \ \ \textbf{break}.\label{line: hat_delta_x}
     }
     \smallgap 
     \texttt{// Phase 2} \\
     \nl \For{$t=t_0+1, \ldots$}{\label{line: phase2_start}
        \smallgap
        \nl Let $p_t =\OP(t,\hatDelta)$ and $\widetilde{p}_t=\frac{1}{2}  e_{\hatx}+\frac{1}{2}p_t$. \label{line: p_tilde}\\
        \nl Sample $x_t\sim \widetilde{p}_t$ and observe $y_t$. \label{line: sample_arm_x}
        
        
        \nl Calculate $\hatell_{t,x}$ and $\hatDelta_{t,x}$ based on \pref{eqn: hatell construction} and \pref{eqn: hat_delta}. \\
        \smallgap
        \nl \textbf{if}
        \vspace{-8pt}
        \begin{align}
             &\exists x\neq \hatx, ~~ \hatDelta_{t,x} \notin \left[0.39\hatDelta_x, 1.81\hatDelta_x\right]\ \ \text{or} \label{eqn: alg_condition_1}\\ 
             &\sum_{s=t_0+1}^t \left(y_s - {\hatell}_{s,\hatx}\right) \geq 20\sqrt{\fT C_1t_0}. \label{eqn: alg_condition_2}
         \end{align}
         \nl \textbf{then break}.\label{line: phase2_end}
     }
     \nl \textbf{Return } $t_0$.
\end{algorithm}

Then, we define the average empirical gap in $[1,t]$ for $x\neq \hatx$ and $t$ in Phase 2 as the following: 
\begin{align}
    \displaystyle\hatDelta_{t,x}&=\frac{1}{t}\left(\sum_{s=1}^{t_0}\hatell_{s,x}+(t-t_0)\Rob_{t,x} - \sum_{s=1}^t \hatell_{s,\hatx}\right)\label{eqn: hat_delta}
\end{align}
where
\begin{align*}
    \Rob_{t,x}=\clip_{[-1,1]}\left(\Catoni_{\alpha_x}\left(\big\{\hatell_{\tau,x}\big\}_{\tau=t_0+1}^{t}\right)\right)  
\end{align*}
with $\alpha_x=\left(\frac{4\log(t|\calX|/\delta)}{t-t_0+\sum_{\tau=t_0+1}^t 2\|x\|_{S_\tau^{-1}}^2}\right)^{\frac{1}{2}}$ (\emph{c.f.} \pref{fig:catoni}). 
Note that we use a simple average estimator for $\hatx$, but a hybrid of average estimator of Phase 1 and robust estimator of Phase 2 for other actions.
These gap estimators are useful in monitoring the non-stochasticity of the environment, which is done via the tests \pref{eqn: alg_condition_1} and \pref{eqn: alg_condition_2}. The first condition (\pref{eqn: alg_condition_1}) checks whether the average empirical gap $\hatDelta_{t,x}$ is still close to the estimated gap $\hatDelta_x$ at the end of Phase 1. The second condition (\pref{eqn: alg_condition_2}) checks whether the regret against $\hatx$ incurred in Phase 2 is still tolerable. 
It can be shown that (see \pref{lem: phase-2-never-end}), with high probability \pref{eqn: alg_condition_1} and \pref{eqn: alg_condition_2} do not hold in a stochastic environment. 
Therefore, when either event is detected, \algse terminates and returns the value of $t_0$ to \alg, which will then run \algse again from scratch with $L=2t_0$.

In the following subsections, we provide a sketch of analysis for \alg, further revealing the ideas behind our design.

\subsection{Analysis for the Adversarial Setting (\pref{thm: fully adversarial})}\label{sec: adv}
We first show that at any time $t$ in Phase 2, with high probability, $\hatx$ is always the best action so far.
\begin{lemma}\label{lem: SAO-argmin-v2}
With probability at least $1-\delta$, for at any $t$ in Phase 2,  we have $\hatx \in \argmin_{x\in\calX} \sum_{s=1}^t \ell_{s,x}$. 
\end{lemma}

\begin{proof}[Proof sketch]
%
The idea is to prove that for any $x\ne \hatx$, the deviation between the actual gap $\sum_{s=1}^t(\ell_{s,x}-\ell_{s, \hatx})$ and the estimated gap $t\hatDelta_{t,x}$ is no larger than $\order(t\hatDelta_x)$.
This is enough to prove the statement since $t\hatDelta_{t,x}$ is of order $\Omega(t\hatDelta_x)$ in light of the test in \pref{eqn: alg_condition_1}.

Bounding the derivation for Phase 2 is somewhat similar to the analysis of \pref{alg:REOLB}, and here we only show how to bound the derivation for Phase 1: $\sum_{s=1}^{t_0}(\ell_{s,x}-\hatell_{s,x})$.
We start by rearranging \pref{eqn: eqL guarantee 1} to get:
$(C_2-1)\left|\sum_{s=1}^{t_0}(\ell_{s,x}-\hatell_{s,x})\right| \leq \sqrt{C_1t_0}-\sum_{s=1}^{t_0}(\ell_{s,x_s}-\hatell_{s,x}) = \sqrt{C_1t_0}-\sum_{s=1}^{t_0}(\ell_{s,x_s}-\hatell_{s,\hatx})+t_0\hatDelta_x$.
By the termination conditions of Phase 1, we have $\sum_{s=1}^{t_0}(\ell_{s,x_s}-\hatell_{s,\hatx})\geq -5\sqrt{f_TC_1t_0}$ and $\hatDelta_x\geq 20\sqrt{f_TC_1/t_0}$, which then shows $\left|\sum_{s=1}^{t_0}(\ell_{s,x}-\hatell_{s,x})\right|\leq \frac{6\sqrt{f_TC_1t_0}+t_0\hatDelta_x}{C_2-1}= \order(t_0\hatDelta_x)$ as desired.
(See \pref{app: adv-setting} for the full proof.)
\end{proof}

We then prove that, importantly, the regret in each epoch is bounded by $\widetilde{\order}(\sqrt{t_0})$ (not square root of the epoch length):

\begin{lemma}\label{lem: adversarial thm}
    With probability at least $1-\delta$, for any time $t$ in Phase 2, we have for any $x\in \calX$, 
    \begin{align*}
         \sum_{s=1}^t\left(\ell_{s,x_s}-\ell_{s,x}\right)=\order\left(\sqrt{C_1t_0f_T}\right).
          \end{align*}
\end{lemma}
\begin{proof}[Proof sketch]
By \pref{lem: SAO-argmin-v2}, it suffices to consider $x = \hatx$.
By \pref{eqn: eqL guarantee 1}, we know that the regret for the first $t_0$ rounds is directly bounded by $\order\left(\sqrt{C_1t_0}\right)$.
For the regret incurred in Phase 2, we decompose it as the sum of $\sum_{s=t_0+1}^t (y_s - \hatell_{s,\hatx})$, 
$\sum_{s=t_0+1}^t  (\hatell_{s,\hatx} -\ell_{s,\hatx} - \epsilon_s(\hatx))$,
and $\sum_{s=t_0+1}^t  (\epsilon_s(\hatx) - \epsilon_s(x_s))$.
The first term is controlled by the test in \pref{eqn: alg_condition_2}.
The second and third terms are martingale difference sequences
with variance bounded by $\order(1-\tildep_{s, \hatx})$, which as we further show is at most $\nicefrac{1}{s\hatDelta_{\min}^2}$ with $\hatDelta_{\min}=\min_{x\neq\hatx} \hatDelta_{x}$.
By combining \pref{eqn:  alg 7 jump condition 1} and \pref{eqn:  alg 7 jump condition 2}, it is clear that $\hatDelta_{\min} \geq 20\sqrt{f_T C_1 / t_0}$ and thus the variance is in the order of $t_0/s$.
Applying Freedman's inequality, the last two terms are thus bounded by $\Otil(\sqrt{t_0})$ as well, proving the claimed result
(see \pref{app: adv-setting} for the full proof).
\end{proof}

Finally, to obtain \pref{thm: fully adversarial}, it suffices to apply \pref{lem: adversarial thm} and the fact that the number of epochs is $\order(\log T)$.
\subsection{Analysis for the Corrupted Setting (\pref{thm: sto-and-corrupt})}\label{sec: sto}

The key for this analysis is the following lemma.

\begin{lemma}\label{lem: phase-2-never-end}
In the corrupted setting, \algse ensures with probability at least $1-15\delta$:
\begin{itemize}
        \item $t_0\leq  \max\left\{\frac{900f_TC_1}{\Delta_{\min}^2}, \frac{900C^2}{f_TC_1}, L\right\}$.    
        \item If $C\leq \frac{1}{30}\sqrt{f_TC_1L}$, then 1) $\hatx=x^*$;
        2) $\hatDelta_x\in  [0.7\Delta_{x}, 1.3\Delta_x]$ for all $x$;
        and 3) Phase 2 never ends.
    \end{itemize}
\end{lemma}

Using this lemma, we show a proof sketch of \pref{thm: sto-and-corrupt} for the stochastic case (i.e. $C=0$). The full proof is deferred to \pref{app: sto-setting}. 

\begin{proof}[Proof sketch for \pref{thm: sto-and-corrupt} with $C=0$] 
By \pref{lem: phase-2-never-end}, we know that after roughly $\Theta\left(\frac{f_TC_1}{\Delta_{\min}^2}\right)$ rounds in Phase~1, the algorithm finds $\hatx=x^*$, estimates $\hatDelta_x$ up to a constant factor of $\Delta_x$, and 
enters Phase~2 without ever going back to Phase~1. By \pref{eqn: eqL guarantee 1}, the regret in Phase 1 can be upper bounded by $\order\left(\sqrt{C_1\cdot \frac{f_TC_1}{\Delta_{\min}^2}}\right)=\order\left(\frac{C_1}{\Delta_{\min}}\sqrt{f_T}\right)$.

To bound the regret in Phase 2, we show that as long as $t$ is larger than a problem-dependent constant $T^*$, there exist $\{N_x\}_{x\in \calX}$ satisfying $\sum_{x\in \calX}N_x\Delta_x\leq 2c(\calX, \theta)$ such that $\{p^*_{t,x}\}_{x\in \calX\setminus\{x^*\}} = \big\{\frac{\htt N_x}{2t}\big\}_{x\in\calX\setminus\{x^*\}}$ is a feasible solution of \pref{eqn: opt-2-constraint}. Therefore, we can bound the regret in this regime as follows:
\begin{align*}
    &\sum_{s=T^*+1}^t\sum_{x\in \calX} \widetilde{p}_{s,x}\Delta_x\\ &= \sum_{s=T^*+1}^t\sum_{x\in \calX} \frac{1}{2} \tag{$x^*=\hatx$} p_{s,x}\Delta_x \\
    &\leq \sum_{s=T^*+1}^t\sum_{x\in \calX} \frac{1}{1.4} p_{s,x}\hatDelta_x \tag{$\hatDelta_x\in [0.7\Delta_x, 1.3\Delta_x]$} \\
    &\leq \sum_{s=T^*+1}^t\sum_{x\in \calX} \frac{1}{1.4} p^*_{s,x}\hatDelta_x \tag{optimality of $p_s$}\\
    &\leq \sum_{s=T^*+1}^t\sum_{x\in \calX} p^*_{s,x}\Delta_x \tag{$\hatDelta_x\in [0.7\Delta_x, 1.3\Delta_x]$}\\
    &\leq \order\left(c(\calX, \theta)\log^2 T\right). \tag{definition of $p_s^*$}
\end{align*}

Combining the regret bounds in Phase~1 and Phase~2, we prove the results for the stochastic setting.
\end{proof}


\section{Conclusion}

In this work, we make significant progress on improving the robustness and adaptivity of linear bandit algorithms.
Our algorithms are the first to achieve near-optimal regret in various different settings, without having any prior knowledge on the environment.
Our techniques might also be useful for more general problems such as linear contextual bandits.

In light of the work~\citep{zimmert2019optimal} for multi-armed bandits that shows a simple Follow-the-Regularized-Leader algorithm achieves optimal regret in different settings, one interesting open question is whether there also exists such a simple Follow-the-Regularized-Leader algorithm for linear bandit with the same adaptivity to different settings.
In fact, it can be shown that their algorithm has a deep connection with \OP in the special case of multi-armed bandits, but we are unable to extend the connection to general linear bandits.

\section*{Acknowledgements}
We thank Tor Lattimore and Julian Zimmert for helpful discussions. HL thanks Ilias Diakonikolas and Anastasia Voloshinov for initial discussions in this direction.
The first four authors are supported by NSF Awards IIS-1755781 and IIS-1943607.

\bibliography{ref}
\bibliographystyle{icml2021}
\onecolumn
\appendix


\section{Auxiliary Lemmas for \OP}

\begin{proof}[\textbf{Proof of \pref{lem: optimization-property-main}}]
    Consider the minimizer $p^*$ of the following constrained minimization problem for some $\xi > 0$:
    \begin{align}
    \min_{p\in \Delta_{\mathcal{X}}} \sum_{x\in \mathcal{X}}p_x\hatDelta_{x}+\frac{2}{\xi}\left(-\ln(\det(S(p)))\right),\label{eq: constrained opt}
    \end{align}
    where $S(p)=\sum_{x\in \mathcal{X}}p_xxx^\top$. We will show that
    \begin{align}
        \sum_{x\in \mathcal{X}}p^*_x\hatDelta_{x} &\le \frac{2d}{\xi}, \label{eq: OP eq 1}\\
        \|x\|_{S(p^*)^{-1}}^2 &\le \frac{\xi \hatDelta_{x}}{2}+d, \;\forall x\in \mathcal{X}.  \label{eq: OP eq 2}
    \end{align}
    To prove this, first note that relaxing the constraints from $p\in \calP_\mathcal{X}$ to the set of sub-distributions $\{p: \sum_{x\in \mathcal{X}}p_x\leq 1 \text{ and } p_x \geq 0,\; \forall x\}$ does not change the solution of this problem.
    This is because for any sub-distribution, we can always make it a distribution by increasing the weight of some $x$ with $\hatDelta_{x}=0$ (at least one exists) while not increasing the objective value (since $\ln(\det(S(p)))$ is non-decreasing in $p_x$ for each $x$).
    Therefore, applying the KKT conditions, we have
    \begin{align}
        \hatDelta_{x} -\frac{2}{\xi}x^\top S(p^*)^{-1}x-\lambda_x+\lambda = 0,\label{eq: KKT}
    \end{align}
    where $\lambda_x, \lambda\ge 0$ are Lagrange multipliers.
    Plugging in the optimal solution $p^*$ and taking summation over all $x\in \mathcal{X}$, we have
    \begin{align*}
        0 &= \sum_{x\in \mathcal{X}}p^*_x\hatDelta_{x}-\frac{2}{\xi}\sum_{x\in \mathcal{X}}p^*_xx^\top S(p^*)^{-1}x -\sum_{x\in \mathcal{X}}\lambda_xp^*_x + \lambda \\
        &= \sum_{x\in \mathcal{X}}p^*_x\hatDelta_{x} - \frac{2}{\xi}\Tr(S(p^*)^{-1} S(p^*))+\lambda \tag{complementary slackness} \\
        &=\sum_{x\in \mathcal{X}}p_x^*\hatDelta_{x}-\frac{2d}{\xi}+\lambda \\
        &\ge \sum_{x\in \mathcal{X}}p_x^*\hatDelta_{x}-\frac{2d}{\xi}\tag{$\lambda\ge 0$}.
    \end{align*}
    Therefore, we have $\sum_{x\in \mathcal{X}}p_x^*\Delta_{x}\leq \frac{2d}{\xi}$ and $\lambda\le \frac{2d}{\xi}$ as $\sum_{x\in \calX}p_x^*\hatDelta_x\geq 0$.
    This proves \pref{eq: OP eq 1}.
    For \pref{eq: OP eq 2}, using \pref{eq: KKT}, we have 
    \begin{align*}
        \|x\|_{S(p^*)^{-1}}^2 = \frac{\xi}{2}\left(\hatDelta_{x}-\lambda_x+\lambda\right) \le \frac{\xi}{2}\left(\hatDelta_{x}+\lambda\right)\le \frac{\xi\hatDelta_{x}}{2}+d,
    \end{align*}
    where the first inequality is due to $\lambda_x\geq 0$, and the second inequality is due to $\lambda\leq \frac{2d}{\xi}$.

    Now we show how to transform $p^*$ into a distribution satisfying the constraint of \OP.
    Choose  $\xi=\frac{\sqrt{t}}{\htt}$.  Let $G = \{x: \hatDelta_{x}\le \frac{1}{\sqrt{t}}\}$. We construct the distribution $q = \frac{1}{2}p^* + \frac{1}{2}q^{G, \delt}$, where $q^{G, \delt}$ is defined in \pref{lem: exist-d-distribution} with $\delt=\frac{1}{\sqrt{t}}$,
    and prove that $q$ satisfies \pref{eqn: opt-2-constraint}.
    Indeed, for all $x\notin G$, we have by definition $\sqrt{t}\hatDelta_{x} \geq 1$ and thus
    \begin{align*}
        \|x\|_{S(q)^{-1}}^2\le \|x\|_{\frac{1}{2}S(p^*)^{-1}}^2 \le \xi\hatDelta_{x}+2d = \frac{\sqrt{t}\hatDelta_{x}}{\htt}+2d\le \frac{t\hatDelta_{x}^2}{\htt}+4d;
    \end{align*}
for $x\in G$, according to \pref{lem: exist-d-distribution} below, we have $\|x\|_{S(q)^{-1}}^2 \leq \|x\|_{\frac{1}{2}S(q^{G, \delt})^{-1}}^2 \leq 4d\leq \frac{t\hatDelta_x^2}{\htt}+4d$.
    
    Combining the two cases above, we prove that $q$ satisfies \pref{eqn: opt-2-constraint}. According to the optimality of $p$, we thus have 
    \begin{align*}
        \sum_{x\in \calX}p_x\hatDelta_x &\leq \sum_{x\in \calX}q_x\hatDelta_x \\
        &= \frac{1}{2}\sum_{x\in\calX}p_x^*\hatDelta_x+\frac{1}{2}\sum_{x\in \calX}q^{G, \delt}_x\hatDelta_x \tag{by the definition of $q$} \\
        &\leq \frac{1}{2}\frac{2d}{\xi} + \frac{1}{2\sqrt{t}}+\frac{1}{2\sqrt{t}} \tag{by \pref{eq: OP eq 1}, $\hatDelta_x \leq 1$, the definition of $G$, and the choice of $\delt$}\\
        &\leq \frac{d\htt+1}{\sqrt{t}} = \order\left(\frac{d\htt}{\sqrt{t}}\right)\tag{by the definition of $\xi$},
    \end{align*}
    proving the lemma.
\end{proof}

The following lemma shows that for any $G\subset \calX$, there always exists a distribution $p\in \calP_\calX$ that puts most weights on actions from $G$, such that $\|x\|_{(\sum_{x\in \calX}p_xxx^\top)^{-1}}^2\le \order(d)$ for all $x\in G$.
\begin{lemma}\label{lem: exist-d-distribution}
    Suppose that $\calX\subseteq \mathbb{R}^d$ spans $\mathbb{R}^d$
    and let $p_{\calX}$ be the uniform distribution over $\calX$.
    For any $G\subseteq \calX$ and $\delt \in (0,\frac{1}{2}]$, there exists a distribution $q\in \calP_{G}$ such that $\|x\|_{S(q^{G, \delt})^{-1}}^2\leq 2d$ for all $x\in G$, where $q^{G, \delt}\triangleq\delt\cdot p_\calX+(1-\delt)\cdot q$ and $S(p) = \sum_{x\in \calX}p_xxx^\top$.
\end{lemma}
\begin{proof}
    Let $\calP_G^{\delt}=\{p\in \calP_{\calX}\;|\; p=\delt\cdot p_{\calX}+(1-\delt)\cdot q, q\in \calP_G\}$. As $\calX$ spans the whole $\mathbb{R}^d$ space, $\|x\|_{S(p)^{-1}}^2$ is well-defined for all $p\in \calP_G^{\delt}$. Then we have
    \begin{align}
        &\min_{p\in \calP_G^\delt} \max_{x\in G} \|x\|_{S^{-1}(p)}^2 \nonumber\\
        &= \min_{p\in \calP_G^\delt} \max_{q\in \calP_G}\left\langle\sum_{x\in \calX}q_xxx^\top, \left(\sum_{x\in \calX}p_xxx^\top\right)^{-1}\right\rangle \nonumber\\
        &=\max_{q\in \calP_G}\min_{p\in \calP_G^{\delt}}\left\langle\sum_{x\in \calX}q_xxx^\top, \left(\sum_{x\in \calX}p_xxx^\top\right)^{-1}\right\rangle \label{eqn: minimax}\\
        &\leq \max_{q\in \calP_G}\left\langle\sum_{x\in \calX}q_xxx^\top, \left(\sum_{x\in \calX}\left(\frac{\delt}{|\calX|}+(1-\delt)q_x\right)xx^\top\right)^{-1}\right\rangle \nonumber\\
        &\leq 2\max_{q\in \calP_G}\left\langle(1-\delt)\sum_{x\in \calX}q_xxx^\top, \left(\sum_{x\in \calX}\left(\frac{\delt}{|\calX|}+(1-\delt)q_x\right)xx^\top\right)^{-1}\right\rangle \tag{$\delt\leq \frac{1}{2}$} \nonumber\\
        &\leq 2\max_{q\in \calP_G}\left\langle\frac{\delt}{|\calX|}\sum_{x\in \calX}xx^\top+(1-\delt)\sum_{x\in \calX}q_xxx^\top, \left(\sum_{x\in \calX}\left(\frac{\delt}{|\calX|}+(1-\delt)q_x\right)xx^\top\right)^{-1}\right\rangle \nonumber\\
        &=2d \nonumber,
    \end{align}
    where the second equality is by the Sion's minimax theorem as \pref{eqn: minimax} is linear in $q$ and convex in $p$.
\end{proof}

\begin{lemma}
    \label{lem: gap-dependent b}
    Given $\{\hatDelta\}_{x\in\calX}$, suppose there exists a unique $\hatx$ such that $\hatDelta_{\hatx}=0$, and $\hatDelta_{\min}=\min_{x\neq \hatx}\hatDelta_{x} >0$. Then $\sum_x p_x\hatDelta_x \le \frac{24d\htt}{\hatDelta_{\min}t}$ when $t\geq \frac{16d\htt}{\hatDelta_{\min}^2}$, where $p$ is the solution to $\OP(t,\hatDelta)$. 
\end{lemma}
\begin{proof}
    We divide actions into groups $G_0, G_1, G_2, \ldots$ based on the following rule:
    \begin{align*}
        &G_0 = \{\hatx\}, \\
        &G_i = \left\{x: 2^{i-1}\hatDelta_{\min}^2\leq \hatDelta_{x}^2<2^{i}\hatDelta_{\min}^2\right\}.
    \end{align*}
     Let $n$ be the largest index such that $G_n$ is not empty and $z_i=\frac{d\htt}{2^{i-2} \hatDelta_{\min}^2 t}$ for $i\geq 1$. For each group $i$, by \pref{lem: exist-d-distribution}, we find a distribution $q^{G_i, \delt}$ with $\delt = \frac{1}{n\cdot 2^n|\calX|}$, such that $\|x\|_{\left(\sum_{y\in \calX}q^{G_i,\delt}_y yy^\top\right)^{-1}}^2\leq 2d$ for all $x\in G_i$. 
    Then we define a distribution $\tildep$ over actions as the following: 
    \begin{align*}
        \tildep_x 
        = \left\{\begin{aligned}
             &\sum_{j\geq 1} z_j q^{G_j,\delt}_x &&\text{if $x\ne \hatx$} \\
            &1-\sum_{x'\neq \hatx} \tildep_{x'} &&\text{if $x=\hatx$.}
        \end{aligned}\right.
    \end{align*}
    $\tildep$ is a valid distribution as
    \begin{align*}
        \tildep_{\hatx}&=1-\sum_{i\geq 1}\sum_{x\in G_i}\sum_{j\geq1}z_j q^{G_j,\delt}_x\\
        &= 1-\sum_{i\geq 1}\sum_{x\in G_i}z_iq^{G_i,\delt}_x - \sum_{i\geq 1}\sum_{x\in G_i}\sum_{j\ne i, j\geq 1}z_jq^{G_j,\delt}_x \\
        &\geq 1-\sum_{i\geq 1}z_i - \sum_{i\geq 1}\sum_{x\in G_i}\sum_{j\ne i, j\geq 1}\frac{z_j}{n\cdot 2^n\cdot |\calX|}\tag{$\sum_{x\in G_i}q^{G_i,\delt}_x\leq 1$ and $q^{G_j,\delt}_x=\frac{1}{n\cdot 2^n|\calX|}$ for $x\notin G_j$}\\
        &\geq 1-\sum_{i\geq 1}z_i - \sum_{i\geq 1}z_i \tag{by definition $\frac{z_j}{2^n}\leq z_i$ for all $i,j$}\\
        &\geq \frac{1}{2}\tag{condition $t\geq \frac{16d\htt}{\hatDelta_{\min}^2}$ and thus $\sum_{i\ge 1}2z_i\le\sum^\infty_{i=1}\frac{2}{2^{i-2}\cdot16}=\frac{1}{2}$}.
    \end{align*}
    Now we show that $\tildep$ also satisfies the constraint of $\OP(t,\hatDelta_x)$. 
    Indeed,
    for any $x\ne \hatx$ and $i$ such that $x\in G_i$, 
    we use the facts $\tildep_y \geq z_iq^{G_i,\kappa}_y$ for $y\neq \hatx$ by definition and $\tildep_{\hatx} \geq \frac{1}{2} \geq \frac{z_i}{n \cdot 2^n |\calX|} = z_iq^{G_i,\kappa}_{\hatx}$ as well to arrive at:
    \begin{align*}
        \|x\|_{S(\tildep)^{-1}}^2 \le \|x\|_{\left(\sum_{y\in \calX} z_iq^{G_i,\kappa}_y yy^\top\right)^{-1}}^2 = \frac{1}{z_i}\|x\|_{\left(\sum_{y\in \calX} q^{G_i,\kappa}_y yy^\top\right)^{-1}}^2  \leq\frac{2d}{z_i}= \frac{2^{i-1}\hatDelta_{\min}^2 t}{\htt} \leq \frac{t\hatDelta_x^2 }{\htt};
    \end{align*}
    for $x=\hatx$, we have $\tildep_{\hatx}\geq \frac{1}{2}$ as shown above and
    thus,
    \begin{align*}
        \|\hatx\|_{S(\tildep)^{-1}}^2=\|S(\tildep)^{-1}\hatx\|_{S(\tildep)}^2\ge \|S(\tildep)^{-1}\hatx\|_{\tfrac{1}{2}\hatx{\hatx}^\top}^2 = \frac{1}{2}\|\hatx\|_{S(\tildep)^{-1}}^4~\Longrightarrow~\|\hatx\|_{S^{-1}(\tildep)}^2\le 2.
    \end{align*}
     
    Thus, $\tildep$ satisfies  \pref{eqn: opt-2-constraint}. Therefore,
    \begin{align*}
        \sum_{x\in\calX}  p_{x}\hatDelta_{x}
        &\le 
        \sum_{x\in\calX}  \tildep_{x}\hatDelta_{x} \tag{by the feasibility of $\tildep$ and the optimality of $p$}\\
        &= \sum_{x \neq \hatx}  \tildep_{x}\hatDelta_{x} \\
        &\le \sum_{i\geq 1}\sum_{x\in G_i}\sum_{j\geq 1} \frac{d\htt}{2^{j-2}\hatDelta_{\min}^2t} q^{G_j,\delt}_x \sqrt{2^i}\hatDelta_{\min} \tag{by the definition of $\tildep_{x}$ and $G_i$}\\
        &= \sum_{i\geq 1}\sum_{x\in G_i}\sum_{j\ne i, j\geq 1} \frac{d\htt q^{G_j,\delt}_x}{2^{j-\nicefrac{i}{2}-2}\hatDelta_{\min}t} + \sum_{i\geq 1}\sum_{x\in G_i}\frac{d\htt q^{G_i,\delt}_x}{2^{\nicefrac{i}{2}-2}\hatDelta_{\min}t}\\
        &\leq \sum_{i\geq 1}\sum_{x\in G_i}\sum_{j\ne i, j\geq 1} \frac{d\htt }{|\calX|\cdot n\cdot 2^{n+j-\nicefrac{i}{2}-2}\hatDelta_{\min}t} + \sum_{i\geq 1}\sum_{x\in G_i}\frac{d\htt q^{G_i,\delt}_x}{2^{\nicefrac{i}{2}-2}\hatDelta_{\min}t} \tag{$q^{G_j,\delt}_x=\frac{1}{n\cdot 2^n\cdot |\calX|}$ for $x\notin G_j$}\\
        &\leq \sum_{i\geq 1}\frac{2d\htt }{2^{\nicefrac{i}{2}-2}\hatDelta_{\min}t}\leq  \frac{24d\htt}{\hatDelta_{\min} t} \tag{$n+j-\frac{i}{2}\geq \frac{i}{2}$},
    \end{align*}
    proving the lemma.
\end{proof}

\begin{lemma}
    \label{lem: close gap lemma looser}
    Let $\hatDelta_x \in \left[\frac{1}{\sqrt{r}}\Delta_x, \sqrt{r}\Delta_x\right]$ for all $x\in\calX$ for some $r>1$, and $p=\OP(t,\hatDelta)$ for some $t\geq \frac{16rd\htt}{\Delta_{\min}^2}$. Then $\sum_{x\in\calX} p_{x}\Delta_x \leq \frac{24rd\htt}{\Delta_{\min} t}$. 
\end{lemma}
\begin{proof}
    By the condition on $\hatDelta_x$, we have $t\geq \frac{16rd\htt}{\Delta_{\min}^2}\geq \frac{16d\htt}{\hatDelta_{\min}^2}$. Also, the condition implies that $\hatDelta_{x^*}=\Delta_{x^*}=0$ and $\hatDelta_{x}>0$ for all $x\neq x^*$. Therefore,  
    \begin{align*}
        \sum_{x\in\calX}p_x\Delta_x \leq \sqrt{r}\sum_{x\in\calX}p_x\hatDelta_x \le \sqrt{r}\frac{24d\htt}{\hatDelta_{\min}t} \le  \frac{24rd\htt}{\Delta_{\min}t},
    \end{align*}
    where the second equality is due to \pref{lem: gap-dependent b} and the other inequalities follow from  $\hatDelta_x \in \left[\frac{1}{\sqrt{r}}\Delta_x, \sqrt{r}\Delta_x\right]$ for all $x$.
 \end{proof}

In the following lemma, we define a problem-dependent quantity $M$. 

\begin{lemma}\label{lem: finite-N-star}
	Consider the optimization problem:
	\begin{align*}
		&\min_{\{N_x\}_{x\in \calX}, N_x\geq 0} \sum_{x\in \calX}N_x\Delta_x \\
		s.t. \qquad &\|x\|_{H(N)^{-1}}^2\leq \frac{\Delta_x^2}{2}, \forall x\in \calX^{-},
	\end{align*}
	where $H(N)=\sum_{x\in \calX}N_xxx^\top$ and $\calX^- = \calX\backslash\{x^*\}$. Define its optimal objective value as $c(\calX, \theta)$ (same as in \pref{sec:prelim}). Then, there exist $\{N^*_x\}_{x\in \calX}$ satisfying the constraint of this optimization problem with $\sum_{x\in \calX} N^*_x\Delta_x\leq 2c(\calX,\theta)$ and $N^*_{x^*}$ being finite.
	(Define $M=\sum_{x\in \calX}N^*_x$.)
\end{lemma}
\begin{proof}
	    If there exists an assignment of $\{N_x\}_{x\in \calX}$ for the optimal objective value which has finite $N_{x^*}$, then the lemma trivially holds. Otherwise, consider the optimal solution $\{\widetilde{N}_x\}_{x\in \calX}$ with $\widetilde{N}_{x^*} = \infty$. According to the constraints, the following holds for all $x\in \calX^{-}$:
	    \begin{align*}
	        \lim_{N\rightarrow \infty}\|x\|_{\left(Nx^*x^{*^\top}+\sum_{y\in \calX^{-}}\widetilde{N}_yyy^\top\right)^{-1}}^2 \leq  \frac{\Delta_x^2}{2}.
	    \end{align*}
	    As $|\calX|$ is finite, by definition, we know for any $\epsilon$, there exists a positive value $M_\epsilon$ such that for all $N\geq M_\epsilon$, $\|x\|^2_{\left(N x^*x^{*^\top}+\sum_{y\in \calX^{-}}\widetilde{N}_yyy^\top\right)^{-1}}\leq \frac{\Delta_x^2}{2}+\epsilon$. Choosing $\epsilon = \frac{\Delta_{\min}^2}{2}$, we have when $N\geq M_\epsilon$, for all $x\in \calX^-$, \begin{align*}
	    \|x\|_{\left(Nx^*x^{*^\top}+\sum_{x\in \calX^{-}}\widetilde{N}_xxx^\top\right)^{-1}}^2 < \frac{\Delta_x^2}{2} + \frac{\Delta_{\min}^2}{2}\leq \Delta_{x}^2.
	    \end{align*}
	    Therefore, consider the solution $\{N^*_x\}_{x\in\calX}$ where $N^*_x=2\widetilde{N}_x$ if $x\in \calX^-$ and $N^*_{x^*}=2M_\epsilon$. We have $\|x\|_{H(N)^{-1}}^2\leq \frac{\Delta_x^2}{2}$. Moreover, the objective value is bounded by $\sum_{x\in \calX}N^*_x\Delta_x = 2\sum_{x\in \calX^{-1}}\widetilde{N}_x\Delta_x = 2c(\calX, \theta)$.
\end{proof}

\begin{lemma}
    \label{lem: close gap lemma}
    Suppose $\hatDelta_x \in  \left[\frac{1}{\sqrt{r}}\Delta_x, \sqrt{r}\Delta_x\right]$ for all $x\in\calX$ for some $r>1$, and $p=\OP(t,\hatDelta)$ for some $t\geq r\htt M$, where $M$ is defined in \pref{lem: finite-N-star}.
    Then $\sum_{x\in\calX} p_{x}\Delta_x \leq \frac{ r^2 \htt}{t}c(\calX,\theta)$. 
\end{lemma}
\begin{proof}
    Recall $N^*$ defined in \pref{lem: finite-N-star}.
    Define $\tildep$, a distribution over $\calX$, as the following: 
    \begin{align*}
        \tildep_x = \begin{cases}
            \frac{r\htt N_x^*}{2t},   &x\neq x^*\\
            1-\sum_{x'\neq x^*}\tildep_{x'}    &x=x^*.
        \end{cases}
    \end{align*}
    It is clear that $\tildep$ is a valid distribution since $t\geq r\htt M$.
    Also, note that by the definition of $M$ and the condition of $t$,
    \begin{align*}
        \tildep_{x^*}=1-\sum_{x'\neq x^*}\frac{r\htt N_x^*}{2t}\geq 1-\frac{r\htt M}{2t}\geq \frac{r\htt M}{2t}\geq \frac{r\htt N_{x^*}^*}{2t}.
    \end{align*}
    Below we show that $\tildep$ satisfies the constraint of $\OP(t,\hatDelta_x)$. 
    Indeed, for any $x\neq x^*$, 
    \begin{align*}
        \|x\|_{S(\tildep)^{-1}}^2 &\le \|x\|_{\left(\sum_{y\in \calX}\frac{r\htt N^*_y}{2t}yy^\top\right)^{-1}}^2\tag{$\tildep_{x^*}\geq \frac{r\htt N_{x^*}^*}{2t}$}\\
        &= \frac{2t}{r\htt}\|x\|_{\left(\sum_{y\in \calX} N^*_y yy^\top\right)^{-1}}^2\\
        &\leq \frac{t\Delta_x^2}{r\htt}\tag{by the constraint in the definition of $\{N_x^*\}_{x\in \calX}$}\\
        &\leq \frac{t\hatDelta_x^2}{\htt}; \tag{$\hatDelta_x \in  \left[\frac{1}{\sqrt{r}}\Delta_x, \sqrt{r}\Delta_x\right]$}
    \end{align*}
    for $x=x^*$, we have $\tildep_{x^*}\geq 1-\frac{r\htt M}{2t}\geq \frac{1}{2}$ by the condition of $t$:
    \begin{align*}
        \|x^*\|_{S(\tildep)^{-1}}^2=\|S(\tildep)^{-1}x^*\|_{S(\tildep)}^2\ge \|S(\tildep)^{-1}x^*\|_{\tfrac{1}{2}x^*{x^*}^\top}^2 = \frac{1}{2}\|x^*\|_{S(\tildep)^{-1}}^4~\Longrightarrow~\|x^*\|_{S^{-1}(\tildep)}^2\le 2.
    \end{align*}
     
    Notice that $\hatDelta_{x^*} \in  \left[\frac{1}{\sqrt{r}}\Delta_{x^*}, \sqrt{r}\Delta_{x^*}\right]$ implies $\Delta_{x^*}=\hatDelta_{x^*}=0$. Thus,  
    \begin{align*}
        \sum_{x\in\calX} p_{x}\Delta_x
        &\le\sqrt{r}\sum_{x\in\calX}  p_{x}\hatDelta_{x}\tag{$\hatDelta_x \in  \left[\frac{1}{\sqrt{r}}\Delta_x, \sqrt{r}\Delta_x\right]$}\\
        &\le\sqrt{r}\sum_{x\in\calX}  \tildep_{x}\hatDelta_{x}\tag{by the feasibility of $\tildep$ and the optimality of $p$}\\
        &\le r\sqrt{r} \htt \sum_{x\in\calX}\frac{N_x^*}{2t}\hatDelta_{x}\tag{by the definition of $\tildep$ and  $\hatDelta_{x^*}=0$}\\
        &\le r^2 \htt \sum_{x\in\calX}\frac{N_x^*}{2t}\Delta_x \tag{$\hatDelta_x \in  \left[\frac{1}{\sqrt{r}}\Delta_x, \sqrt{r}\Delta_x\right]$}\\
        &\le \frac{r^2 \htt}{t}c(\calX,\theta)\tag{$\sum_x N_x^*\Delta_x\leq 2c(\calX,\theta)$ proven in \pref{lem: finite-N-star}}, 
    \end{align*}
    finishing the proof.
\end{proof}

\begin{lemma}\label{lem: opt-constant}
    We have $c(\calX,\theta) \leq \frac{48d}{\Delta_{\min}}$.
\end{lemma}

\begin{proof}
The idea is similar to that of \pref{lem: gap-dependent b}.
    Define $G_0=\{x^*\}$, $G_i=\left\{x: \Delta_x^2\in [2^{i-1}, 2^i)\Delta_{\min}^2\right\}$ and $n$ be the largest index such that $G_n$ is not empty. For each $i\geq 1$, let $q^{G_i, \delt}\in \calP_\calX$ with $\delt = \frac{1}{|\calX|\cdot n\cdot 2^{n}}$ be the distribution such that $\|x\|_{S(q^{G_i, \delt})^{-1}}^2\leq 2d$ for all $x\in G_i$ (see \pref{lem: exist-d-distribution}). Let $N_{x^*}=\infty$ and for $x\in G_i$, we let $N_{x}=\sum_{j\geq 1}\frac{4dq^{G_j,\delt}_x}{2^{j-1}\Delta_{\min}^2}$. Next we show that $\{N_x\}_{x\in \calX}$ satisfies  the constraint \pref{eqn: opt-c-con-main}.
    
    In fact, fix $x\in G_i\subseteq \calX^-$, by definition of $\{N_x\}_{x\in \calX}$, we have
    \begin{align*}
        \|x\|_{\left(\sum_{x\in \calX}N_xxx^\top\right)^{-1}}^2\leq \|x\|_{\left(\sum_{x\in \calX}\frac{4dq^{G_i,\delt}_xxx^\top}{2^{i-1}\Delta_{\min}^2}\right)^{-1}}^2 = \frac{2^{i-1}\Delta_{\min}^2}{4d}\|x\|_{S(q^{G_i, \delt})^{-1}}^2\leq 2^{i-2}\Delta_{\min}^2 \leq \frac{\Delta_x^2}{2},
    \end{align*}
    where the first inequality is because $S(q^{G_i, \delt})$ is invertible. 
    Therefore, the objective value of \pref{eqn: opt-c-obj-main} is bounded as follows:
    \begin{align*}
        \sum_{x\in \calX}N_x\Delta_x
        &=\sum_{i\geq 1}\sum_{x\in G_i}\sum_{j\geq 1}\frac{4dq^*_{G_j,x}}{2^{j-1}\Delta_{\min}^2}\Delta_x \tag{by the definition of $N_x$}\\
        &\leq \sum_{i\geq 1}\sum_{x\in G_i}\sum_{j\geq 1}\frac{4dq^*_{G_j,x}}{2^{j-\nicefrac{i}{2}-1}\Delta_{\min}} \tag{by the definition of $G_i$}\\
        & = \sum_{i\geq 1}\sum_{x\in G_i}\sum_{j\ne i, j\geq 1}\frac{4dq^*_{G_j,x}}{2^{j-\nicefrac{i}{2}-1}\Delta_{\min}} + \sum_{i\geq 1}\sum_{x\in G_i} \frac{4dq^{G_i,\delt}_x}{2^{\nicefrac{i}{2}-1}\Delta_{\min}}\\
        & \leq \sum_{i\geq 1}\sum_{x\in G_i}\sum_{j\ne i, j\geq 1}\frac{4d}{|\calX|\cdot n\cdot 2^{n+j-\nicefrac{i}{2}-1}\Delta_{\min}} + \sum_{i\geq 1} \frac{4d}{2^{\nicefrac{i}{2}-1}\Delta_{\min}} \tag{$q^{G_j,\delt}_x = \frac{1}{|\calX|\cdot n\cdot 2^n}$ for $j\ne i$, $x\in G_i$}\\
        &\leq \sum_{i\geq 1}\frac{4d}{2^{n-\nicefrac{i}{2}-1}\Delta_{\min}} + \sum_{i\geq 1} \frac{4d}{2^{\nicefrac{i}{2}-1}\Delta_{\min}} \tag{$j\geq 1$} \\
        &\leq \sum_i \frac{8d}{2^{\nicefrac{i}{2}-1}\Delta_{\min}} \leq \frac{48d}{\Delta_{\min}} \tag{$i\leq n$}. 
    \end{align*}
    Therefore, we have $c(\calX,\theta)\leq \frac{48d}{\Delta_{\min}}$.
\end{proof}

\section{Analysis for \pref{alg:REOLB}}\label{app: sto}
In this section, we analyze the performance of \pref{alg:REOLB} in the stochastic or corrupted setting. For \pref{thm: sto-alg-guarantee}, we decompose the proof into two parts. First, we show in \pref{lem: regret_part2} (a more concrete version of \pref{lem: regret_part2-main}) that for some constant $T^*$ specified later, we have $\sum_{t=T^*+1}^T p_{t,x}\Delta_x=\order(c(\calX, \theta)\log T\log(T|\calX|/\delta))$. Second, using \pref{lem: optimization-property-main}, we know that \pref{alg:REOLB} also enjoys a regret bound of $\order(d\sqrt{T}\log(T|\calX|/\delta)+C)$, which gives $\order(d\sqrt{T^*}\log(T^*|\calX|/\delta)+C)$ for the first $T^*$ rounds and proves \pref{thm: regret_part1-main}.

To prove \pref{lem: regret_part2}, we first show in \pref{lem:twice-main} that $\Delta_x$ and $\hatDelta_{m,x}$ are close within some multiplicative factor with some additional terms related to the corruption. This holds with the help of \pref{eqn: opt-2-constraint} and the use of robust estimators. For notational convenience, first recall the following definitions from \pref{lem:twice-main}.
\begin{definition}
    \begin{align*}
        \hmm &\triangleq  \optconst\log(2^m|\calX|/\delta)=\gtt_{2^m}, \\
        C_m&\triangleq \sum_{\tau\in\calB_m} \max_x |c_{\tau,x}|, \\
        \rho_m &\triangleq 
        \begin{cases}
            \sum_{k=0}^m\frac{2^kC_k}{4^{m-1}}, &\mbox{if $m\geq 0$},\\
            0, &\mbox{$m=-1$}.
        \end{cases}
    \end{align*}
\end{definition}


\begin{proof}[\textbf{Proof of \pref{lem:twice-main}}]
    We prove this by induction. For base case $m=0$, $\hatDelta_{m,x}= 0$ by definition, and 
    Also, $\Delta_x\le 2\le \sqrt{\frac{d\hmm}{4}}$ and \pref{eqn: hyp1-double} holds. 
 
    If both inequalities hold for $m$, then
    \begin{align}
        \|x\|_{S_m^{-1}}^2 &\le \frac{2^m\hatDelta_{m,x}^2}{\hmm}+4d 
        \le \frac{2^m\left(2\Delta_x+\sqrt{\frac{d \hmm}{4\cdot 2^m}} + 2\rho_{m-1}\right)^2}{\hmm} + 4d 
        \le \frac{16\cdot 2^m\Delta_x^2}{\hmm }+ 8d + \frac{16\cdot 2^m \rho_{m-1}^2}{\hmm}, \label{eqn: expansing variance single step}
    \end{align}
    where the first inequality is because of \pref{eqn: opt-2-constraint} of $\OP(2^m, \hatDelta_{m,x})$.
    Next, we show that the expectation of $\ellhat_{\tau,x}$ is $\inner{x, \theta+c_\tau}$ and the variance of $\ellhat_{\tau,x}$ is upper bounded by $\|x\|_{S_m^{-1}}^2$ for $\tau\in \calB_m$. In fact, we have 
    \begin{align}
        &\mathbb{E}\left[\ellhat_{\tau,x}\right] = \mathbb{E}\left[x^\top S_m^{-1}x_\tau \left(\inner{x_\tau, \theta+c_\tau}+\epsilon_\tau\right)\right] = \inner{x,\theta+c_\tau},\\
        &\mathbb{E}\left[\hatell_{\tau,x}^2\right] \leq \mathbb{E}[x^\top S_{m}^{-1}x_{\tau}x_{\tau}^\top S_{m}^{-1}x\cdot y_\tau^2]
        \le x^\top S_{m}^{-1}x = \|x\|_{S_m^{-1}}^2.   \label{eqn: variance upper bound single step}
    \end{align}
    
    Now we are ready to prove the relation. Let $\bar{c}_{m,x}=\frac{1}{2^{m}}\sum_{\tau\in\calB_m} c_{\tau,x}$. Under the induction hypothesis for the case of $m$, using \pref{lem: catoni concentration} with $\mu_i = \inner{x, \theta+c_i}$, with probability at least $1-\frac{\delta}{4^{m}}$, we have for all $x\in \calX$:
    \begin{align*}
        &\Delta_x - \hatDelta_{m+1,x}\\
        &= \inner{x, \theta} - \inner{x^*, \theta} - \Rob_{m,x} + \min_x \Rob_{m,x}\\
        &\leq \inner{x, \theta} - \inner{x^*, \theta} - \Rob_{m,x} + \Rob_{m,x^*} \\
        &\leq \left|\Rob_{m,x^*}-\inner{x^*, \theta}-\frac{\sum_{\tau\in \calB_m}\inner{x^*,c_\tau}}{2^m}\right| + \left|\Rob_{m,x}-\inner{x,\theta}-\frac{\sum_{\tau\in \calB_m}\inner{x, c_{\tau}}}{2^m}\right| + \frac{2C_m}{2^m}\\
        &\leq \frac{1}{2^m}\left(\alpha_x\left(2^m\|x\|^2_{S_m^{-1}}+\sum_{\tau\in \calB_m}(c_{\tau,x}-\bar{c}_{m,x})^2\right)+\frac{4\log(2^m|\calX|/\delta)}{\alpha_x}\right)\\
        &\quad + \frac{1}{2^m}\left(\alpha_{x^*}\left(2^m\|x^*\|^2_{S_m^{-1}}+\sum_{\tau\in \calB_m}(c_{\tau,x^*}-\bar{c}_{m,x^*})^2\right)+\frac{4\log(2^m|\calX|/\delta)}{\alpha_{x^*}}\right)+\frac{2C_m}{2^m} \tag{by \pref{eqn: variance upper bound single step} and \pref{lem: catoni concentration}}\\
        &\leq \frac{1}{2^m}\left(\alpha_x\left(2^{m}\|x\|_{S_m^{-1}}^2+2^m\right)+ \frac{\hmm}{2^{12}\alpha_x} \right) +  \frac{1}{2^m}\left(\alpha_{x^*}\left(2^{m}\|x^*\|_{S_m^{-1}}^2+2^m\right)+\frac{\hmm}{2^{12}\alpha_{x^*}}\right) +\frac{2C_m}{2^m}\tag{using the definition of $\gamma_m$ and $\sum_{\tau\in\calB_m}(c_{\tau,x}-\bar{c}_{m,x})^2 \leq \sum_{\tau\in\calB_m}c_{\tau,x}^2\leq 2^m$}\\
        &= \frac{2}{64\cdot 2^m}\sqrt{\left(2^{m}\|x\|_{S_m^{-1}}^2+2^{m}\right)\hmm } + \frac{2}{64\cdot 2^m}\sqrt{\left(2^{m}\|x^*\|_{S_m^{-1}}^2+2^{m}\right)\hmm } + \frac{C_m}{2^{m-1}}   \tag{by the choice of $\alpha_x$}\\
       &\leq \frac{4}{64\cdot 2^m}\sqrt{\left(\frac{16\cdot 2^{2m}\Delta_x^2}{\hmm }+16\cdot 2^m d + \frac{16\cdot 2^{2m}\rho_{m-1}^2}{\hmm }\right)\hmm } + \frac{C_m}{2^{m-1}} \tag{by \pref{eqn: expansing variance single step}} \\
       &\leq \frac{\Delta_x}{2} + \sqrt{\frac{d\hmm }{16\cdot 2^m}} + \frac{\rho_{m-1}}{4} + \frac{C_m}{2^{m-1}}  \tag{$\sqrt{a+b}\leq \sqrt{a}+\sqrt{b}$} \\
       &\leq \frac{\Delta_x}{2} + \sqrt{\frac{d\hmm }{16\cdot 2^m}} + \rho_m. \tag{by definition of $\rho_m$}
    \end{align*}
    Therefore, $\Delta_x\le 2\hatDelta_{m+1,x}+\sqrt{\frac{d\hmm }{4\cdot 2^m}} + 2\rho_m$, 
    which proves \pref{eqn: hyp1-double}. The other claim \pref{eqn: hyp2-double} can be proven using similar analysis. Taking a union bound over all $m$ finishes the proof.
\end{proof}



\begin{lemma}[A detailed version of \pref{lem: regret_part2-main}]\label{lem: regret_part2}
    Let $T^*\triangleq \frac{32C}{\Delta_{\min}} +  4M'\log\left(\frac{2M'|\calX|}{\delta}\right)$, where $M'=2^{20}\left(M+\frac{d}{\Delta_{\min}^2}\right)$ and $M$ is defined in \pref{lem: finite-N-star}. Then \pref{alg:REOLB} guarantees with probability at least $1-\delta$:
    \begin{align*}
        \sum^{T}_{t=T^*+1}\sum_xp_{m_t,x}\Delta_x\le \order\left( c(\calX, \theta) \log(T)\log(T/\delta)\right).
    \end{align*}
\end{lemma}
\begin{proof}
    First, note that according to \pref{lem: tech-2}, $t\geq 4M'\log\left(\frac{2M'|\calX|}{\delta}\right)$ implies $t\ge M'\log\left(\frac{t|\calX|}{\delta}\right)\geq \frac{32 d\htt}{\Delta_{\min}^2}$. 
    Let $m_t$ be the epoch such that $t\in \calB_{m_t}$, which means $t\in [2^{m_t}, 2^{m_t+1})$ and thus $\htt\geq \gmm_{m_t}$. Then we have,
    \begin{align}
        \sqrt{\frac{d\gmm_{m_t}}{4\cdot 2^{m_t}}}+2\rho_{m_t} &\leq \sqrt{\frac{d\gmm_{m_t}}{4\cdot 2^{m_t}}} + \frac{2C}{2^{m_t-1}} \leq \frac{1}{4}\Delta_{\min}+ \frac{1}{4}\Delta_{\min}\leq \frac{1}{2}\Delta_{x}, \label{eqn: additive-term-vanish}
    \end{align}
    where the first inequality is by definition of $\rho_m$ and the second inequality is because $2^{m_t+1}\geq t\geq T^*\geq \frac{32C}{\Delta_{\min}}$ and $t\geq \frac{32d\gamma_{m_t}}{\Delta_{\min}^2}$. 
    Then by \pref{lem:twice-main}, we have for all $x\neq x^*$, $\Delta_{x} \leq 2\hatDelta_{m_t,x} + \frac{\Delta_x}{2}$, $\hatDelta_{m_t,x} \leq 2\Delta_{x} + \frac{\Delta_x}{2}$,
    which gives
    \begin{align}
        \Delta_{x} \leq 4\hatDelta_{m_t,x},\qquad \hatDelta_{m_t,x} \leq 4\Delta_{x} \label{eqn: pure-multiplicative-bound}.
    \end{align}
    Since $\hatDelta_{m_t,x}\geq \frac{1}{4}\Delta_x$ for all $x\neq x^*$, we must have $\hatDelta_{m_t,x^*}=\Delta_{x^*}=0$. 
    Moreover, according to the definition of $T^*$, we have $t\geq 2^{20}M\log(\frac{t|\calX|}{\delta})\geq 16\beta_t M$. Therefore, the conditions of \pref{lem: close gap lemma} hold. Applying \pref{lem: close gap lemma} with $r=16$, we get 
    \begin{align*}
        \sum_x p_{m_t,x}\Delta_x\leq \order\left(\frac{\htt}{2^{m_t}}c(\calX,\theta)\right). 
    \end{align*}
    Summing over $t\geq T^*+1$, we get with probability at least $1-\delta$,
    \begin{align*}
        \sum_{t=T^*+1}^T \sum_x p_{m_t,x}\Delta_x= \order\left( \htT c(\calX,\theta)\log (T)\right) = \order\left(c(\calX,\theta)\log(T)\log(T|\calX|/\delta)\right). 
    \end{align*}

\end{proof}

Now we are ready to prove the main result \pref{thm: sto-alg-guarantee}.



\begin{proof}[\textbf{Proof of \pref{thm: sto-alg-guarantee}}]

    Let $\E_t[\cdot]$ denote the conditional expectation given the history up to time $t$.
    By Freedman's inequality, we have
    \begin{align}
        \Reg(T) &= \sum_{t=1}^T\sum_{x\in \calX}\one\{x_t=x\}\Delta_x \nonumber\\
        &\leq\sum_{t=1}^T\sum_{x\in \calX}p_{m_t,x}\Delta_x+ 2\sqrt{\log(1/\delta)\sum_{t=1}^T\E_t\left[\left(\sum_{x\ne x^*}\one\{x_t=x\}\Delta_x\right)^2\right]}+\log(1/\delta) \nonumber\\
        &\leq\sum_{t=1}^T\sum_{x\in \calX}p_{m_t,x}\Delta_x+ 2\sqrt{\log(1/\delta)\sum_{t=1}^T\E_t\left[\sum_{x\ne x^*}\one\{x_t=x\}\Delta_x\right]}+\log(1/\delta) \tag{$\Delta_x\leq 1$}\\
        &\leq \sum_{t=1}^T\sum_{x\in \calX}p_{m_t,x}\Delta_x+ 2\sqrt{\log(1/\delta)\sum_{t=1}^T\sum_{x\in \calX}p_{m_t,x}\Delta_x}+\log(1/\delta) \nonumber\\
        &\leq 2\sum_{t=1}^T\sum_{x\in\calX}p_{m_t, x}\Delta_x+2\log(1/\delta).   \label{eqn: relating true and fake regret}
    \end{align}

    Fix a round $t$ and the corresponding epoch $m_t$. According to \pref{lem: optimization-property-main}, we know that 
    \begin{align*}
        \sum_{x\in \calX}p_{m_t,x}\hatDelta_{m_t,x}\leq \order\left(\frac{d\log(2^{m_t}|\calX|/\delta)}{2^{\nicefrac{m_t}{2}}}\right)
    \end{align*}
    
    Therefore, combining \pref{lem:twice-main}, the regret in epoch $m$ is bounded by
    
\begin{align*}
    2^m\sum_{x\in\calX} p_{m,x}\Delta_x  &\leq 2^m\left(2\sum_{x} p_{m,x}\widehat{\Delta}_{m,x} + \order\left(\sqrt{ \frac{d \hmm }{2^m}} + \rho_{m-1}\right)\right) \\
    &\leq \order\left(d\cdot 2^{\nicefrac{m}{2}}\log(2^m/\delta)\right) + \order\left(\sqrt{2^md\hmm }+\sum_{k=0}^{m-1} \frac{2^kC_k}{4^{m-2}} \right) \\
    &\leq \order\left(d\cdot 2^{\nicefrac{m}{2}}\log(2^m/\delta)+\sqrt{2^md\hmm }+\sum_{k=0}^m \frac{C_k}{2^{m-k}}\right).
\end{align*}
    Summing up the regret till round $T$, we have
    \begin{align}
        \sum_{t=1}^T \sum_{x\in \mathcal{X}}p_{m_t,x}\Delta_{x}   
        &\leq \sum_{m=0}^{\log_2 T}\order\left(d\cdot 2^{\nicefrac{m}{2}}\log(2^m|\calX|/\delta)+\sqrt{2^md\hmm }\right) + \order\left(\sum_{k=0}^{\log_2 T}C_k\sum_{m=k}^{\log_2 T}\frac{1}{2^{m-k}}\right) \nonumber  \\
        &= \order\left(d\sqrt{T}\log(T|\calX|/\delta)+\sum_{k=0}^{\log_2 T}C_k\right) \nonumber  \\
        &= \order\left(d\sqrt{T}\log(T|\calX|/\delta)+C\right).   \label{eqn: small T*}
    \end{align}

    Using \pref{eqn: small T*} for $t\leq T^*$ and using \pref{lem: regret_part2} for the rest, we know that
    \begin{align*}
        \sum_{t=1}^T\sum_{x\in \calX}p_{m_t,x}\Delta_x \leq \order\left(c(\calX,\theta)\log(T)\log(T|\calX|/\delta)+d\sqrt{T^*}\log(T^*|\calX|/\delta) + C\right).
    \end{align*}
    Combining this with \pref{eqn: relating true and fake regret}, we get
    \begin{align}
        \Reg(T) 
        &= \order\left(\sum_{t=1}^T\sum_{x\in \calX}p_{m_t,x}\Delta_x + \log(1/\delta)\right) \nonumber\\
        &=\order\left(c(\calX, \theta)\log(T)\log(T|\calX|/\delta)+d\sqrt{T^*}\log(T^*|\calX|/\delta)+C\right)\nonumber\\
        &=\order\left(c(\calX, \theta)\log(T)\log(T|\calX|/\delta) + d\sqrt{\frac{C}{\Delta_{\min}}}\log\left(\frac{C|\calX|}{\Delta_{\min}\delta}\right)+ d\sqrt{M'\log\left(\frac{M'|\calX|}{\delta}\right)}\log\left(\frac{M'|\calX|}{\delta}\right)  + C  \right) \nonumber\\
        &=\order\left(c(\calX, \theta)\log(T)\log(T|\calX|/\delta) + d\sqrt{\frac{C}{\Delta_{\min}}}\log\left(\frac{C|\calX|}{\Delta_{\min}\delta}\right)  + C + M^*\left(\log(1/\delta)\right)^{\frac{3}{2}} \right)\label{eqn: M-star}
    \end{align}
    for some constant $M^*$ that depends on $M'=2^{18}\left(M+\frac{d}{\Delta_{\min}^2}\right)$ and $\log|\calX|$. 
\end{proof}

In the following, we prove an alternative bound of $\order(\frac{d^2(\log T)^2}{\Delta_{\min}}+C)$, which is independent of $M^*$. The following lemma is an analogue of \pref{lem: regret_part2}, but the constant $T'$ is independent of $M^*$. 

\begin{lemma}\label{lem: regret_part2_noM}
    Let $T'\triangleq \frac{32C}{\Delta_{\min}} +  \frac{2^{25}d}{\Delta_{\min}^2}\log\left(\frac{2^{24}d|\calX|}{\delta\Delta_{\min}^2}\right)$. Then \pref{alg:REOLB} guarantees with probability $1-\delta$
    \begin{align*}
        \sum^{T}_{t=T'+1}\sum_{x\in \calX}p_{m_t,x}\Delta_x\le \order\left(\frac{d\log(T)\log(T|\calX|/\delta)}{\Delta_{\min}}\right).
    \end{align*}
\end{lemma}

\begin{proof}
    First, note that according to \pref{lem: tech-2}, $t\geq \frac{2^{25}d}{\Delta_{\min}^2}\log\left(\frac{2^{24}d|\calX|}{\Delta_{\min}^2\delta}\right)$ implies $t\ge \frac{2^{23}d}{\Delta_{\min}^2}\log\left(\frac{t|\calX|}{\delta}\right)$. 
    Let $m_t$ be the epoch such that $t\in \calB_{m_t}$, which means $t\in [2^{m_t}, 2^{m_t+1})$ and thus $\htt\geq \gmm_{m_t}$. Therefore, we still have \pref{eqn: additive-term-vanish}, which further shows that we have \pref{eqn: pure-multiplicative-bound} and $\hatDelta_{m_t,x^*}=\Delta_{x^*}=0$. 
    Moreover, according to the definition of $T'$, we have $t\geq \frac{2^{23}d}{\Delta_{\min}^2}\log(\frac{t|\calX|}{\delta})\geq 256d\beta_t/\Delta_{\min}^2$. Therefore, the conditions of \pref{lem: close gap lemma looser} hold. Applying \pref{lem: close gap lemma looser} with $r=16$, we get 
    \begin{align*}
        \sum_x p_{m_t,x}\Delta_x\leq \order\left(\frac{d\htt}{2^{m_t}\Delta_{\min}}\right). 
    \end{align*}
    
    Summing over $t\geq T'+1$, we get with probability at least $1-\delta$,
    \begin{align*}
        \sum_{t=T'+1}^T \sum_x p_{m_t,x}\Delta_x= \order\left( \frac{d\htT\log (T)}{\Delta_{\min}}\right) = \order\left(\frac{d\log(T)\log(T|\calX|/\delta)}{\Delta_{\min}}\right). 
    \end{align*}
\end{proof}

\begin{theorem}\label{thm: no-M-dependency}
\pref{alg:REOLB} guarantees that with probability at least $1-\delta$,
\begin{align*}
    \Reg(T) &= \order\left(\frac{d^2}{\Delta_{\min}}\log^2\left(\frac{T|\calX|}{\Delta_{\min}\delta}\right)+C \right).
\end{align*}
\end{theorem}
\begin{proof}

    
    

    Using \pref{eqn: small T*} for $t\leq T'$ and using \pref{lem: regret_part2_noM} for the rest, we know that
    \begin{align*}
        \sum_{t=1}^T\sum_{x\in \calX}p_{m_t,x}\Delta_x \leq \order\left(\frac{d\log(T)\log(T|\calX|/\delta)}{\Delta_{\min}}+d\sqrt{T'}\log(T'|\calX|/\delta) + C\right).
    \end{align*}
    Combining this with \pref{eqn: relating true and fake regret}, we get
    \begin{align*}
        \Reg(T) 
        &= \order\left(\sum_{t=1}^T\sum_{x\in \calX}p_{m_t,x}\Delta_x + \log(1/\delta)\right) \\
        &=\order\left(\frac{d\log(T)\log(T|\calX|/\delta)}{\Delta_{\min}}+d\sqrt{T'}\log(T'|\calX|/\delta)+C\right)\\
        &=\order\left(\frac{d\log(T)\log(T|\calX|/\delta)}{\Delta_{\min}} + d\sqrt{\frac{C}{\Delta_{\min}}}\log\left(\frac{C|\calX|}{\Delta_{\min}\delta}\right)+ d\sqrt{\frac{d}{\Delta_{\min}^2}\log\left(\frac{d|\calX|}{\delta\Delta_{\min}^2}\right)}\log\left(\frac{d|\calX|}{\delta\Delta_{\min}^2}\right)  + C  \right) \\
        &=\order\left(\frac{d\log(T)\log(T|\calX|/\delta)}{\Delta_{\min}} + d\sqrt{\frac{C}{\Delta_{\min}}}\log\left(\frac{C|\calX|}{\Delta_{\min}\delta}\right)  + C + \frac{d^{\frac{3}{2}}\log^{\frac{3}{2}}(\frac{d|\calX|}{\delta\Delta_{\min}})}{\Delta_{\min}} \right) \\
        &\leq \order\left(\frac{d^2}{\Delta_{\min}}\log^2\left(\frac{T|\calX|}{\Delta_{\min}\delta}\right)+C\right) \tag{using AM-GM inequality and $C\leq T$, $T\geq d$}.
    \end{align*}
\end{proof}

Finally, we prove \pref{thm: regret_part1-main}, which is a direct result by combining \pref{eqn: small T*} and \pref{eqn: relating true and fake regret}.
\begin{proof}[\textbf{Proof of \pref{thm: regret_part1-main}}]
    Combining \pref{eqn: small T*} and \pref{eqn: relating true and fake regret}, we have
    \begin{align*}
        \Reg(T) = \order\left(\sum_{t=1}^T\sum_{x\in \calX}p_{m_t,x}\Delta_x+\log(1/\delta)\right) =\order\left(d\sqrt{T}\log(T|\calX|/\delta)+C\right).
    \end{align*}
\end{proof}

\section{Analysis of \pref{alg: BOBW}}\label{app: adv-setting}




In this section, we show that \pref{alg: BOBW} achieves both minimax-optimality in the adversarial setting and near instance-optimality in the stochastic setting. In \pref{app: adv-setting-sub}, we prove \pref{thm: fully adversarial}, showing that \pref{alg: BOBW} enjoys $\Otil(\sqrt{T})$ regret in the adversarial setting. In \pref{app: sto-setting}, we prove \pref{thm: sto-and-corrupt}, showing that \pref{alg: BOBW} also enjoys nearly instance-optimal regret in the stochastic setting and slightly worse regret in the corrupted setting.

\subsection{Analysis of \pref{alg: BOBW} in the adversarial setting}\label{app: adv-setting-sub}

To prove the guarantee in the adversarial setting, we first prove \pref{lem: SAO-argmin-v2}, which shows that at any time in Phase 2, $\hatx$ has the smallest cumulative loss within $[1,t]$. 


\begin{proof}[\textbf{Proof of \pref{lem: SAO-argmin-v2}}]
    By \pref{assum:  adversarial alg}, for any $x$, and any $t$ in Phase 1, 
    \begin{align*}
        \sum_{s=1}^{t}(\ell_{s,x_s} - \ell_{s,x})
        &\leq \sqrt{C_1t} - C_2\left|\sum_{s=1}^{t}(\ell_{s,x} - \hatell_{s,x})\right| \\
        &\leq \sqrt{C_1t} - (C_2-1)\left|\sum_{s=1}^{t}(\ell_{s,x} - \hatell_{s,x})\right| - \sum_{s=1}^{t}(\ell_{s,x} - \hatell_{s,x}), 
    \end{align*}
    which implies 
    \begin{align}
        \left\vert \sum_{s=1}^{t} (\ell_{s,x} - \hatell_{s,x})  \right\vert &\le \frac{1}{C_2-1}\left(\sqrt{C_1t}  - \sum_{s=1}^{t}(\ell_{s,x} - \hatell_{s,x})- \sum_{s=1}^{t}(\ell_{s,x_s} - \ell_{s,x})\right)\nonumber\\
        &\leq \frac{1}{C_2-1}\left(\sqrt{C_1t} + \sum_{s=1}^t \hatell_{s,x} - \sum_{s=1}^t  \ell_{s,x_s}\right). \label{eqn: guarantee 3}
    \end{align}
    At time $t_0$, we have with probability at least $1-2\delta$,
    \begin{align}
        \left\vert \sum_{s=1}^{t_0} (\ell_{s,x} - \hatell_{s,x})  \right\vert 
        &\leq \frac{1}{C_2-1}\left(\sqrt{C_1t_0} + \sum_{s=1}^{t_0} \hatell_{s,x} - \sum_{s=1}^{t_0} \ell_{s,x_s}\right)    \tag{by \pref{eqn: guarantee 3}}\\
        &\leq \frac{1}{C_2-1}\left(2\sqrt{C_1t_0} + \sum_{s=1}^{t_0} \hatell_{s,x} - \sum_{s=1}^{t_0} y_s\right)    \tag{by Azuma's inequality}\\
        &\leq \frac{1}{C_2-1}\left(2\sqrt{C_1t_0}+5\sqrt{f_TC_1t_0} + \sum_{s=1}^{t_0} \hatell_{s,x} - \sum_{s=1}^{t_0} \hatell_{s,\hatx}\right)  \tag{by \pref{eqn: alg 7 jump condition 1}} \\
        &= \frac{1}{C_2-1}\left(7\sqrt{f_TC_1t_0} + t_0 \hatDelta_x \right).   \label{eqn: useful3} 
   \end{align}

   
\textbf{Bounding the deviation of $(t-t_0)\Rob_{t,x}$ for $x\ne \hatx$:}
   
For all $x\ne \hatx$, the variance of $\ellhat_{\tau,x}$ is bounded as follows:
\begin{align}
    \text{Var}(\ellhat_{\tau,x})\leq \mathbb{E}\left[\ellhat_{\tau,x}^2\right] \leq \E\left[(x^\top\widetilde{S}_{\tau}^{-1}x_tx_t^\top\widetilde{S}_{\tau}^{-1}x)^2\right]\leq \|x\|_{\widetilde{S}_\tau^{-1}}^2\leq 2\|x\|_{S_{\tau}^{-1}}^2,   \label{eqn: tmp variance bound}
\end{align}
where the last inequality is due to $\widetilde{S}_\tau=\frac{1}{2}\hatx\hatx^\top+\frac{1}{2}S_\tau\succeq \frac{1}{2}S_\tau$. Therefore, using \pref{lem: catoni concentration} with $\mu_i =\ell_{i, x}$, with probability at least $1-2\delta$, for all $t$ in Phase 2 and all $x\ne \hatx$, 
\begin{align}
    \left|(t-t_0)\cdot \Rob_{t,x}-\sum_{\tau=t_0+1}^t\ell_{\tau,x}\right| &\leq \alpha_x\left(2\sum_{\tau=t_0+1}^t\|x\|_{S_\tau^{-1}}^2 + \sum_{\tau=t_0+1}^t\left(\ell_{\tau,x}-\frac{1}{t-t_0}\sum_{\tau'=t_0+1}^t\ell_{\tau',x}\right)^2\right) + \frac{2\log\frac{t^2|\calX|}{\delta}}{\alpha_x}\nonumber\\
    &\leq \alpha_x\sum_{\tau=t_0+1}^t\left(2\|x\|_{S_{\tau}^{-1}}^2+1\right) + \frac{4\log\frac{t|\calX|}{\delta}}{\alpha_x}\nonumber\\
    &\leq 2\sqrt{4\sum_{\tau=t_0+1}^t \left(2\|x\|_{S_\tau^{-1}}^2+1\right)\log\frac{t|\calX|}{\delta}} \tag{Choose $\alpha_x$ optimally}\\
    &\leq 2\sqrt{4\log\frac{t|\calX|}{\delta}\sum_{\tau=t_0+1}^t\left(\frac{2\tau\hatDelta_x^2}{\beta_\tau}+9d\right)} \label{eqn: catoni-con},
\end{align}
where the last inequality is due to \pref{eqn: opt-2-constraint}. For $\tau\ge t_0$, since
\begin{align}
    \hatDelta_x = \frac{1}{t_0}\left(\sum_{s=1}^{t_0}\ellhat_{s,x}-\ellhat_{s,\hatx}\right)\geq 20\sqrt{\frac{f_TC_1}{t_0}} \label{eqn: hatDelta-geq-sqrt},
\end{align}
we have 
\begin{align*}
    \hatDelta_x \geq 20\sqrt{\frac{f_TC_1}{\tau}}\ge20\sqrt{\frac{f_Td\htT}{\tau}}\ge 20\sqrt{\frac{d\htT}{\tau}}\ge 20\sqrt{\frac{d\beta_\tau}{\tau}}\ge 3\sqrt{\frac{d\beta_\tau}{2\tau}}
\end{align*}
and $9d\leq \frac{2\tau\hatDelta_x^2}{\beta_\tau}$. Note that $h(\tau)=\frac{\tau}{\log(\tau|\calX|/\delta)}$ an increasing function when $\delta\leq 0.1$.
Using \pref{eqn: catoni-con} and $9d\leq \frac{2\tau\hatDelta_x^2}{\beta_\tau}$, we have
\begin{align}
    \left|(t-t_0)\cdot \Rob_{t,x}-\sum_{s=t_0+1}^t\ell_{s,x}\right|\leq 2\sqrt{4\log\frac{t|\calX|}{\delta}\sum_{\tau=t_0+1}^t\frac{4\tau\hatDelta_x^2}{\beta_\tau}}\leq 2\sqrt{16t^2\hatDelta_x^2\log\frac{t|\calX|}{\delta}\frac{1}{\beta_t}} \leq \frac{t\hatDelta_x}{16}.\label{eqn: rob_l_gap}
\end{align}

For the first $t_0$ rounds, according to \pref{eqn: useful3}, we have
\begin{align}
    \left|\sum_{s=1}^{t_0}(\ell_{s,x}-\hatell_{s,x})\right| \leq \frac{1}{C_2-1}\left(7\sqrt{f_TC_1t_0}+t_0\hatDelta_x\right)\leq  \frac{1.4t_0\hatDelta_x}{C_2-1},\label{eqn: l_hatl_gap}
\end{align}
where the last inequality is due to \pref{eqn: hatDelta-geq-sqrt}.
Combining \pref{eqn: rob_l_gap} and \pref{eqn: l_hatl_gap} and noticing that $C_2\geq 20$, we have for all $x\ne \hatx$,
\begin{align}\label{eqn: sub-opt-x-est}
    \left|\sum_{s=1}^{t_0}(\ell_{s,x}-\ellhat_{s,x})+\sum_{s=t_0+1}^t\left(\ell_{s,x}-\Rob_{t,x}\right)\right| \leq \frac{1.7t\hatDelta_x}{10}.
\end{align}

\textbf{Bounding the deviation of $\sum_{s=1}^t\ellhat_{s,\hatx}$ (recall that we use the standard average estimator for $\hatx$):}

For the first $t_0$ rounds, according to \pref{eqn: useful3}, since $\hatDelta_{\hatx}=0$, we have
\begin{align*}
    \left|\sum_{s=1}^{t_0}(\ell_{s,\hatx}-\ellhat_{s,\hatx})\right| \leq \sqrt{f_TC_1t_0}.
\end{align*}

For $t\geq t_0+1$, according to Freedman's inequality and the fact that $\E_t\left[\hatell_{s,\hatx}^2\right] = \mathbb{E}_t\left[\frac{y_s^2}{\widetilde{p}^2_{s,\hatx}}\cdot\one\{x_s=\hatx\}\right]\leq \frac{1}{\widetilde{p}_{s,\hatx}}\leq 2$ as $\widetilde{p}_{s,\hatx}\geq \frac{1}{2}$, we have with probability at least $1-\delta$,
\begin{align*}
    \left|\sum_{s=t_0+1}^t \ell_{s,\hatx} - \sum_{s=t_0+1}^t\hatell_{s,\hatx}\right|\leq 2\sqrt{2t\log(t|\calX|/\delta)} + 2\log(t|\calX|/\delta) \leq \sqrt{C_1t}.
\end{align*}
Combining the above two inequalities, we have
\begin{align}\label{eqn: opt-x-est}
        \left|\sum_{s=1}^{t}(\ell_{s,\hatx}-\ellhat_{s,\hatx})\right| \leq 3\sqrt{f_TC_1t}.
\end{align}

\textbf{In sum:}
combining the bounds for $x\ne \hatx$ and $x=\hatx$, we have for all $x\ne \hatx$,
\begin{align*}
    \sum_{s=1}^t(\ell_{s,x}-\ell_{s,\hatx})
    &\geq \sum_{s=1}^{t-1}(\ell_{s,x}-\ell_{s,\hatx})-2 \\
    &\geq \sum_{s=1}^{t_0}\left(\ellhat_{s,x}-\ellhat_{s,\hatx}\right) + \left((t-t_0-1)\Rob_{t-1,x} - \sum_{s=t_0+1}^{t-1}\hatell_{s,\hatx}\right) \\
    &\qquad - 3\sqrt{f_TC_1(t-1)} - \frac{1.7(t-1)\hatDelta_x}{10} - 2   \tag{\pref{eqn: sub-opt-x-est} and \pref{eqn: opt-x-est}}\\
    &\geq (t-1)\hatDelta_{t-1,x}-4\sqrt{f_TC_1(t-1)}-\frac{1.7(t-1)\hatDelta_x}{10}  \tag{by the definition of $\hatDelta_{t-1,x}$ in \pref{eqn: hat_delta}} \\
    &\geq (t-1)\hatDelta_{t-1,x}-\frac{3.7(t-1)\hatDelta_x}{10}    \tag{by \pref{eqn: hatDelta-geq-sqrt}}\\
    &\geq 0.02(t-1)\hatDelta_x > 0,
\end{align*}
where the last inequality is because $t$ belongs to Phase $2$, which means that at time $t-1$, \pref{eqn: alg_condition_1} is satisfied.
\end{proof}

Now we are ready to prove our main lemma in the adversarial setting.

\begin{proof}[\textbf{Proof of \pref{lem: adversarial thm}}]
    By \pref{lem: SAO-argmin-v2}, we know that the regret comparator is $\hatx$. By the regret bound of $\calA$ and the fact that $t_0\ge L_0$ (recall $L_0$ from \pref{assum:  adversarial alg}), we have 
    \begin{align*}
        \sum_{s=1}^{t_0}(\ell_{s,x_s} - \ell_{s,\hatx}) \leq O\left(\sqrt{C_1t_0}\right).  
    \end{align*}
    For the regret in Phase $2$, first note that it suffices to consider $t$ not being the last round of this phase (since the last round contributes at most $2$ to the regret). Then, consider the following decomposition: 
    \begin{align*}
        \sum_{s=t_0+1}^t(\ell_{s,x_s}-\ell_{s,\hatx}) 
        &\leq \sum_{s=t_0+1}^{t-1}(y_s-\epsilon_s(x_s)-\ell_{s,\hatx})+2\\ &=\underbrace{\sum_{s=t_0+1}^{t-1} \left(y_s - \hatell_{s,\hatx}\right)}_{\term{1}} + \underbrace{\sum_{s=t_0+1}^{t-1} \left( \hatell_{s,\hatx} -\ell_{s,\hatx} - \epsilon_s(\hatx)\right)}_{\term{2}} + \underbrace{\sum_{s=t_0+1}^{t-1} \left( \epsilon_s(\hatx) - \epsilon_s(x_s) \right)}_{\term{3}} + 2.  
    \end{align*}
    
    \term{1} is upper bounded by $\order\left(\sqrt{f_T C_1 t_0}\right)$ since it corresponds to the termination condition \pref{eqn: alg_condition_2}. 
    
    \term{2} is a martingale difference sequence since 
    \begin{align*}
        \E_t\left[ \hatell_{s,\hatx} -\ell_{s,\hatx} - \epsilon_s(\hatx) \right] 
        = \E_t\left[\frac{(\ell_{s,\hatx} + \epsilon_s(\hatx))\mathbb{I}\{x_s=\hatx\}}{\tildep_{s,\hatx}} - (\ell_{s,\hatx} + \epsilon_s(\hatx)) \right] = 0. 
    \end{align*}
    The variance is upper bounded by 
    \begin{align}
        \E_t\left[\left(\hatell_{s,\hatx} -\ell_{s,\hatx} - \epsilon_s(\hatx)\right)^2\right]&= \E_t\left[(\ell_{s,\hatx}+\epsilon_s(\hatx))^2\left(\frac{\mathbb{I}\{x_s=\hatx\}}{\tildep_{s,\hatx}} - 1\right)^2\right] \nonumber\\
        &\leq \tildep_{s,\hatx}\left(\frac{1}{\tildep_{s,\hatx}}-1\right)^2 + (1-\tildep_{s,\hatx})\nonumber\\
        &\leq 2(1-\tildep_{s,\hatx}),
        \label{eqn: reuse calculation 1}
    \end{align}
    where the last term is because $\tildep_{s,\hatx}\geq \frac{1}{2}$.
    
    \term{3} is also a martingale difference sequence. As $\epsilon_s(x)\in [-2,2]$, its variance can be upper bounded by 
    \begin{align}
        \E_t\left[ \left(\epsilon_s(\hatx) - \epsilon_s(x_s)\right)^2 \right] \leq 16\E_t\left[\mathbb{I}\{x_s\neq \hatx\}\right] = 16(1-\tildep_{s,\hatx}). \label{eqn: reuse calculation 2}
    \end{align}
    
    Therefore, with probability at least $1-\delta/t$, we have $\term{2}+\term{3}=\order\left(\sqrt{\sum_{s=t_0+1}^t (1-\tildep_{s,\hatx})\log(t/\delta)} + \log(t/\delta)\right)$ by Freedman's inequality.

    
    As $p_t=\OP(t, \hatDelta)$ and $t\geq t_0\geq \frac{400C_1\fT}{\hatDelta_{\min}^2}\geq \frac{16d\htt}{\hatDelta_{\min}^2}$, by \pref{lem: gap-dependent b}, we have 
    \begin{align}
        1-\tildep_{t,\hatx} = \frac{1}{2}\left(1-p_{t,\hatx}\right) \leq \frac{1}{2}\sum_x p_{t,x}\frac{\hatDelta_x}{\hatDelta_{\min}} \leq \frac{12d\htt}{t\hatDelta_{\min}^2}. \label{eqn: 1-p bound}  
    \end{align}

    Combining the above with $\hatDelta_{\min}\geq 20\sqrt{\frac{C_1\fT}{t_0}}$, we get 
    \begin{align*}
        \term{2}+\term{3} 
        &= \order\left(\sqrt{\log(t/\delta)\sum_{s=t_0+1}^t \frac{d\htt}{s\hatDelta_{\min}^2}}+ \log(t/\delta)\right) \\
        &= \order\left(\sqrt{\frac{dt_0\htt\log(t/\delta)\log(t)}{f_TC_1}}+\log(t/\delta)\right) \\
        &= \order\left(\sqrt{t_0\log(t/\delta) } + \log(t/\delta)\right)
    \end{align*}
    where the last step uses the definition of $\htt$ and $C_1\geq \optconst d\log(T|\calX|/\delta)$ from \pref{assum:  adversarial alg}.
    
    Combining all bounds above, we have shown
    \begin{align*}
        \sum_{s=1}^{t}(\ell_{s,x_s} - \ell_{s,\hatx}) = \order\left(\sqrt{C_1t_0f_T}\right),
    \end{align*}
    proving the lemma.
 \end{proof}   
    

\pref{thm: fully adversarial} can then be proven by directly applying \pref{lem: adversarial thm} to each epoch and using the fact that the number of epochs is at most $\order(\log T)$.

\subsection{Analysis of \pref{alg: BOBW} in the corrupted stochastic setting}\label{app: sto-setting}


In this section, we prove our results in the corrupted setting. To prove the main lemma \pref{lem: phase-2-never-end}, we separate the proof into two parts, \pref{lem: phase-1-v1} and \pref{lem: phase-2-no-end-v2}.

\begin{lemma}\label{lem: phase-1-v1}
    In the stochastic setting with corruptions, within a single epoch,
    \begin{enumerate}
        \item with probability at least $1-4\delta$,  $t_0\leq   \max\left\{\frac{900f_TC_1}{\Delta_{\min}^2}, \frac{900C^2}{f_TC_1}, L\right\}$;    
        \item if $C\leq \frac{1}{30}\sqrt{f_TC_1L}$, then with probability at least $1-\delta$, $\hatx=x^*$; 
        \item if $C\leq \frac{1}{30}\sqrt{f_TC_1L}$, then with probability at least $1-2\delta$, $t_0\geq \frac{64f_TC_1}{\Delta_{\min}^2}$;
        \item if $C\leq \frac{1}{30}\sqrt{f_TC_1L}$, then with probability at least $1-3\delta$, $\hatDelta_x\in  [0.7\Delta_{x}, 1.3\Delta_x]$ for all $x\neq x^*$.
    \end{enumerate}
\end{lemma}
\begin{proof}
    In the corrupted setting, we can identify $\ell + c_t$ as $\ell_{t}$ in the adversarial setting.  
    We first show the following property: at any $t$ in Phase $1$ and with probability at least $1-\delta$, for any $x$,
    \begin{align}
        C_2\left|\sum_{s=1}^t (\ell_{s,x}-\hatell_{s,x})\right| \leq \sqrt{C_1t} + t\Delta_x + 2C. \label{eqn: bound dev}
    \end{align}
    By the guarantee of $\calA$, we have with probability at least $1-\delta$, for any $x$ and $t\in [T]$
    \begin{align}
        C_2\left|\sum_{s=1}^t (\ell_{s,x}-\hatell_{s,x})\right| \leq \sqrt{C_1t} + \sum_{s=1}^t \ell_{s,x} - \sum_{s=1}^t \ell_{s,x_s}.   \label{eqn: useful 7}
    \end{align}
    Since $\ell_{t,x_t}\geq \ell_{t,x^*}-\max_{x\in \calX}|c_{t,x}|$, we have for any $t$,
    \begin{align}
        \sum_{s=1}^t \ell_{s,x_s} \geq \sum_{s=1}^t \ell_{s,x^*} - C.  \label{eqn: useful2}
    \end{align}
    Combining \pref{eqn: useful 7} and \pref{eqn: useful2}, and using $\ell_{s,x}-\ell_{s,x^*}\leq \Delta_x + \max_{x'\in \calX}|c_{s,x'}|$ for any $x\in \calX$, we get
    \pref{eqn: bound dev}. 
    
    Below, we define $\dev_{t,x}\triangleq\left|\sum_{s=1}^t (\ell_{s,x}-\hatell_{s,x}) \right|$. 
    
    \paragraph{Claim 1's proof:}
    Let $t= \max\left\{\frac{900f_TC_1}{\Delta_{\min}^2}, \frac{900C^2}{f_TC_1}, L\right\}$. Below we prove that if Phase 1 has not finished before time $t$, then for the choice of $\hatx=x^*$, both \pref{eqn: alg 7 jump condition 1} and \pref{eqn: alg 7 jump condition 2} hold with high probability at time $t$. 
    
    Consider \pref{eqn:  alg 7 jump condition 1}. With probability at least $1-2\delta$, 
    \begin{align*}
        \sum_{s=1}^t y_s 
        &\geq \sum_{s=1}^t  \ell_{s,x_s} - \sqrt{C_1t} \tag{by Azuma's inequality} \\
        &\geq
        \sum_{s=1}^t \ell_{s,x^*} - \sqrt{C_1t} - C \tag{by \pref{eqn: useful2}} \\
        &\geq \sum_{s=1}^t \hatell_{s,x^*} - 2\sqrt{C_1t} - 3C \tag{by \pref{eqn: bound dev} and $\Delta_{x^*}=0$} \\ &\geq  \sum_{s=1}^t \hatell_{s,x^*} - 3\sqrt{f_TC_1t}, \tag{$t\ge \frac{900C^2}{f_TC_1}$ and $\sqrt{f_TC_1t} \geq 30C$}
    \end{align*}
    
    showing that \pref{eqn: alg 7 jump condition 1} holds for $\hatx=x^*$. 
    
    For \pref{eqn: alg 7 jump condition 2}, by the regret bound of $\calA$, with probability at least $1-2\delta$, for $x\neq x^*$,  
    \begin{align*}
        &\sum_{s=1}^t y_s - \sum_{s=1}^t \hatell_{s,x}\\ 
        &= \sum_{s=1}^t (y_s - \ell_{s,x_s}) +  \sum_{s=1}^t (\ell_{s,x_s} - \ell_{s,x^*}) + \sum_{s=1}^t (\ell_{s,x^*} - \ell_{s,x}) + \sum_{s=1}^t (\ell_{s,x}-\hatell_{s,x})  \\
        &\leq \sqrt{C_1t} + \left(\sqrt{C_1t} - C_2\dev_{t,x^*}\right) + \left( -t\Delta_x + C\right) + \dev_{t,x}  \tag{by the regret bound of $\calA$ and Azuma's inequality}\\
        &\leq \left(2 + \frac{1}{30}\right)\sqrt{f_TC_1t} -t\Delta_x + \frac{1}{C_2}\left(\sqrt{C_1t} + t\Delta_x + 2C\right)  \tag{by \pref{eqn: bound dev} and that $30C\leq \sqrt{f_TC_1t}$}\\
        &\leq -0.95 t\Delta_x + 2.1\sqrt{f_TC_1t}.   \tag{$C_2\geq 20$}
    \end{align*} 
    By the condition of $t$, we have $t\Delta_x \geq 30\sqrt{f_TC_1t}$ for all $x\neq x^*$. 
    Thus, the last expression can further be upper bounded by $(-30\times 0.95+2.1)\sqrt{f_TC_1t}\leq -25\sqrt{f_TC_1t}$, indicating that \pref{eqn: alg 7 jump condition 2} also holds for all $x \neq x^*$. Combining the two parts above finishes the proof.

    \paragraph{Claim 2's proof:} Note that \pref{eqn: alg 7 jump condition 1} and \pref{eqn: alg 7 jump condition 2} jointly imply that
    \begin{align}
        \sum_{s=1}^{t_0} (\hatell_{s,x} - \hatell_{s,\hatx}) \geq 20\sqrt{f_TC_1t_0}\qquad  \forall x\neq \hatx.  \label{eqn: useful contradiction}
    \end{align}
    However, with probability at least $1-\delta$, for any $x\neq x^*$,  
    \begin{align*} 
        \sum_{s=1}^{t_0} (\hatell_{s,x^*} - \hatell_{s,x}) 
        &= \sum_{s=1}^{t_0} (\hatell_{s,x^*} - \ell_{s,x^*}) + \sum_{s=1}^{t_0} (\ell_{s,x^*} - \ell_{s,x}) + \sum_{s=1}^{t_0} (\ell_{s,x} - \hatell_{s,x}) \\
        &\leq \dev_{t_0,x^*} +\left(- t_0\Delta_x + C\right) + \dev_{t_0,x} \\
        &\leq \frac{1}{C_2}\left(\sqrt{C_1t_0} + 2C\right) + \left(- t_0\Delta_x +  C\right) + \frac{1}{C_2}\left(\sqrt{C_1t_0} + t_0\Delta_x + 2C\right)   \tag{by \pref{eqn: bound dev}} \\
        &\leq 5\sqrt{f_TC_1t_0}.     \tag{using $C\leq \frac{1}{30}\sqrt{f_TC_1t_0}$ and $C_2\geq 20$}
    \end{align*}
    Therefore, to make \pref{eqn: useful contradiction} hold, it must be that $\hatx=x^*$. 
    
    \paragraph{Claim 3's proof:} Suppose that $t_0\leq \frac{64f_TC_1}{\Delta_{\min}^2}$, and let $x$ be such that $\Delta_x=\Delta_{\min}$. Then we have
    \begin{align*}
        \sum_{s=1}^{t_0} \left(\hatell_{s,x} - \hatell_{s,x^*}\right) 
        &\leq \sum_{s=1}^{t_0} \left(\ell_{s,x} - \ell_{s,x^*}\right) + \dev_{t_0,
        x} + \dev_{t_0,x^*} \tag{hold  w.p. $1-\delta$}\\
        &\leq \left(t_0\Delta_{\min} + C\right) +  \frac{1}{C_2}\left(2\sqrt{C_1t_0} + t_0\Delta_{\min} + 4C\right)   \tag{hold  w.p. $1-\delta$ by \pref{eqn: bound dev}}\\
        &\leq 2t_0\Delta_{\min} + 2\sqrt{f_TC_1t_0}   \tag{by $C\leq \frac{1}{30}\sqrt{f_TC_1t_0}$ and $C_2=20$}\\
        &\leq 16\sqrt{f_TC_1t_0} + 2\sqrt{f_TC_1t_0} \tag{$t_0\leq \frac{64f_TC_1}{\Delta_{\min}^2}$}\\
        &= 18\sqrt{f_TC_1t_0}.
    \end{align*}
    Recall that \pref{eqn: useful contradiction} needs to hold, and recall from Claim 2 that $\hatx=x^*$ holds with probability $1-\delta$. Thus, the bound above is a contradiction. Therefore, with probability $1-2\delta$, $t_0\geq \frac{64f_TC_1}{\Delta_{\min}^2}$.
    
    \paragraph{Claim 4's proof. }
    For notational convenience, denote the set $[a-b,a+b]$ by $[a\pm b]$.
    We have
    \begin{align*}
        t_0\hatDelta_{x} 
        &= \sum_{s=1}^{t_0} \left(\hatell_{s,x} - \hatell_{s,x^*}\right) \in \left[\sum_{s=1}^{t_0} (\ell_{s,x}-\ell_{s,x^*}) \pm (\dev_{t_0,x} + \dev_{t_0,x^*})\right] \\
        &\subseteq \left[t_0\Delta_x \pm \left(C + \frac{1}{C_2}\left(2\sqrt{C_1t_0} + t_0\Delta_x + 4C\right)\right)\right] \tag{hold w.p. $1-\delta$ by \pref{eqn: bound dev}} \\
        &\subseteq \left[t_0\Delta_x \pm \left(\frac{1}{C_2}t_0\Delta_x + \sqrt{f_TC_1t_0}\right)\right]    \tag{using $C\leq \frac{1}{30}\sqrt{f_TC_1t_0}$ and $C_2\geq 20$}\\
        &\subseteq \left[t_0\Delta_x \pm \left(\frac{1}{C_2}t_0\Delta_x + \frac{1}{8}t_0\Delta_x\right)\right]  \tag{by Claim 3, $t_0\geq \frac{64f_TC_1}{\Delta_{\min}^2}$ holds w.p. $1-2\delta$} \\
        &\subseteq \left[\left(1\pm 0.3\right)t\Delta_x\right], 
    \end{align*}
    which finishes the proof.
    \end{proof}
    
    The next lemma shows that when $L$ grows large enough compared to the total corruption $C$, the termination condition \pref{eqn: alg_condition_2} will never be satisfied once the algorithm enters Phase $2$.
    
\begin{lemma}\label{lem: phase-2-cond-2-v2}
  \pref{alg: BOBW} guarantees that with probability at least $1-10\delta$, for any $t$ in Phase $2$, when $0\leq C\leq \frac{1}{30}\sqrt{f_TC_1L}$, we have
    \begin{align*}
        \sum_{s=t_0+1}^t\left(y_s-\ellhat_{s,\hatx}\right) &\leq 20\sqrt{\fT C_1t_0}. \\
    \end{align*}
    Furthermore, when $t\geq M'=10\htt M$ ($M$ is the constant defined in \pref{lem: finite-N-star}), 
    we have
    \begin{align*}
         \sum_{s=t_0+1}^t\left(y_s-\ellhat_{s,\hatx}\right) &=\order\left(c(\calX, \theta)\log T\log(T|\calX|/\delta) + \sqrt{f_TC_1t_0}+d\gtt_{M'}\sqrt{M'}\right).
    \end{align*}
\end{lemma}

\begin{proof}
Recall that $y_s = \ell_{s,x_s} + \epsilon_s(x_s)$ and $\hatell_{s,\hatx}=\frac{\ell_{s,\hatx} + \epsilon_s(\hatx)}{\tildep_{s,\hatx}}\one\{x_s=\hatx\}$.  Thus, 
    \begin{align}
        &\sum_{s=t_0+1}^t\left(y_s-\ellhat_{s,\hatx}\right) \nonumber \\
        &= \underbrace{\sum_{s=t_0+1}^t  \left(\ell_{s,x_s} - \ell_{s, \hatx} \right)}_{\term{1}}  + \underbrace{\sum_{s=t_0+1}^t\left(\ell_{s,\hatx}-\frac{\ell_{s,\hatx}}{\widetilde{p}_{s,\hatx}}\one\{x_s=\hatx\}\right)}_{\term{2}} + \underbrace{\sum_{s=t_0+1}^t\left(\epsilon_s(x_s)-\epsilon_s(\hatx)\right)}_{\textsc{Term 3}} + \underbrace{\sum_{s=t_0+1}^t\left(\epsilon_s(\hatx)-\frac{\epsilon_s(\hatx)}{\widetilde{p}_{s,\hatx}}\one\{x_s=\hatx\}\right)}_{\textsc{Term 4}} \nonumber\\
    \end{align}
    Except for \term{1}, all terms are martingale difference sequences. Let $\E_0$ be the expectation taken over the randomness before Phase 2. Similar to the calculation in \pref{eqn: reuse calculation 1} and \pref{eqn: reuse calculation 2}, we have
    \begin{align*}
        \E_s\left[\left(\ell_{s,\hatx}-\frac{\ell_{s,\hatx}}{\widetilde{p}_{s,\hatx}}\one\{x_s=\hatx\}\right)^2\right]\leq 2\E_0[1-\tildep_{s,\hatx}], \quad \qquad 
        \E_s\left[\left(\epsilon_s(\hatx)-\frac{\epsilon_s(\hatx)}{\widetilde{p}_{s,\hatx}}\one\{x_s=\hatx\}\right)^2\right]\leq 8\E_0[1-\tildep_{s,\hatx}]
    \end{align*}
    and
    \begin{align*}
        \E_s\left[\left(\epsilon_s(x_s)-\epsilon_s(\hatx)\right)^2\right]\leq 16\E_0\left[1-\tildep_{s,\hatx}\right].
    \end{align*}
    By Freedman's inequality, we have with probability at least $1-3\delta$, for all $t$ in Phase $2$,
    \begin{align*}
        &\term{2}+\term{3}+\term{4}\\
        &\leq 2\sqrt{2\sum_{s=s_0+1}^t\E_0[1-\widetilde{p}_{s,\hatx}]\log(T/\delta)} + \log(T/\delta) + 2\sqrt{16\sum_{s=s_0+1}^t\E_0[1-\widetilde{p}_{s,\hatx}]\log(T/\delta)} + 4\log(T/\delta) \\
        &\qquad + 2\sqrt{8\sum_{s=s_0+1}^t\E_0[1-\widetilde{p}_{s,\hatx}]\log(T/\delta)} + 2\log(T/\delta) \\
        &\leq 20\sqrt{\sum_{s=s_0+1}^t\E_0[1-\widetilde{p}_{s,\hatx}]\log(T/\delta)} + 7\log(T/\delta).
    \end{align*}

    Then we deal with \term{1}. Again, by Freeman's inequality with probability at least $1-\delta$, for all $t$ in Phase 2,   
    \begin{align*}
        \sum_{s=t_0+1}^t (\ell_{s,x_s} - \ell_{s,\hatx}) 
        &\leq \sum_{s=t_0+1}^t \sum_{x\neq \hatx}\tildep_{s,x}(\ell_{s,x} - \ell_{s,\hatx}) + 4\sqrt{\sum_{s=t_0+1}^t\E_0 [1-\tildep_{s,\hatx}]\log(T/\delta)} + 2\log(T/\delta) \\
        & \leq C+\sum_{s=t_0+1}^t \sum_{x\neq \hatx}\tildep_{s,x}(\Delta_{x} - \Delta_{\hatx}) + 4\sqrt{\sum_{s=t_0+1}^t \E_0[1-\tildep_{s,\hatx}]\log(T/\delta)} + 2\log(T/\delta)\\
        &\leq C+\frac{1}{2}\sum_{s=t_0+1}^t \sum_{x\neq \hatx}p_{s,x}\Delta_{x} + 4\sqrt{\sum_{s=t_0+1}^t \E_0[1-\tildep_{s,\hatx}]\log(T/\delta)} + 2\log(T/\delta).\tag{$\widetilde{p}_{s,x}=\frac{1}{2}p_{s,x}$ for $x\ne \hatx$}
    \end{align*}
    
    When $C\in [0, \frac{1}{30}\sqrt{f_TC_1L}]\subseteq [0, \frac{1}{30}\sqrt{f_TC_1t_0}]$, according to \pref{lem: phase-1-v1}, we know that with probability $1-4\delta$, $\hatx=x^*$ and $\hatDelta_x\in [0.7\Delta_x, 1.3\Delta_x]$.

    Also by \pref{lem: phase-1-v1}, with probability $1-2\delta$, for any $s\geq t_0$, we have $s\geq t_0\geq \frac{64f_TC_1}{\Delta_{\min}^2}\geq \frac{48d\gtt_s}{\Delta_{\min}^2}$. These conditions satisfy the requirement in \pref{lem: close gap lemma looser} with $r=3$. Therefore we can apply \pref{lem: close gap lemma looser} and get 
    \begin{align}
        \sum_{x\in\calX} p_{s,x}\Delta_x \leq \frac{72d\gtt_s}{\Delta_{\min} s}\label{eqn: p2-regret}
    \end{align}
    for all $s\geq t_0$. Combining all the above, we get 
    \begin{align*}
        \sum_{s=t_0+1}^t\left(y_s - \hatell_{s,\hatx}\right) 
        = C + \frac{72d\htt\log t}{\Delta_{\min}}+ 24\sqrt{\sum_{s=t_0+1}^t \E_0[1-\tildep_{s,\hatx}]\log(T/\delta)} + 9\log(T/\delta).
    \end{align*}
    As argued in \pref{eqn: 1-p bound}, $1-\tildep_{s,\hatx}\leq \frac{12d\gtt_s}{\hatDelta_{\min}^2s}$. Therefore, the above can be further upper bounded by 
    \begin{align*}
        \sum_{s=t_0+1}^t\left(y_s - \hatell_{s,\hatx}\right)
        &\leq C + \frac{96d\htt\log t}{\hatDelta_{\min}} + 24\sqrt{\sum_{s=t_0+1}^t\frac{12d\gtt_s}{\hatDelta_{\min}^2s}\log (T/\delta)} \tag{$\hatDelta_x\in [0.7\Delta_x, 1.3\Delta_x]$, $1-\widetilde{p}_{s,\hatx}\leq \frac{12d\gtt_s}{\hatDelta_{\min}^2s}$} \\
        &\leq C + \frac{96d\htt\log t}{\hatDelta_{\min}} + \frac{144d\beta_T\log T}{\hatDelta_{\min}} \tag{by definition of $\htT$}\\
        &\leq C + 10\sqrt{\frac{t_0}{f_TC_1}}d\htT \log T \tag{$t_0\geq \frac{64\fT C_1}{\Delta_{\min}^2} \geq \frac{24\fT C_1}{\hatDelta_{\min}^2}$}\\
        &\leq \frac{1}{30}\sqrt{\fT C_1t_0} + 10\sqrt{\frac{t_0}{f_TC_1}}d\htT \log T \tag{$C\leq \frac{1}{30}\sqrt{f_T C_1t_0}$}\\
        &\leq 20\sqrt{f_TC_1t_0}. \tag{$C_1\geq d\htT$}
    \end{align*}
    Below, we use an alternative way to bound $\sum_{x\in \calX}p_{s,x}\Delta_x$. 
    Let $M'\geq 20\gtt_{M} M$, which implies $M'\geq 10\gtt_{M'} M$. For $s\in[t_0+1, M']$, we use \pref{lem: optimization-property-main}, and bound 
    \begin{align}\label{eqn: sto-p2}
        \sum_{s=t_0+1}^{M'}\sum_{x\in\calX} p_{s,x}\Delta_x \leq \frac{1}{0.7}\sum_{s=t_0+1}^{M'}\sum_{x\in\calX} p_{s,x}\hatDelta_x \leq \frac{1}{0.7}\sum_{s=t_0+1}^{M'}\frac{d\gtt_s}{\sqrt{s}} \leq \order\left(d\gtt_{M'}\sqrt{M'}\right).
    \end{align}
    For $s>M'$, we use \pref{lem: close gap lemma} and bound
    \begin{align}\label{eqn: sto-p3}
        \sum_{s=M'+1}^{t}\sum_{x\in\calX}p_{s,x}\Delta_x \leq \sum_{s=M'+1}^t \order\left(\frac{\gtt_s}{s}c(\calX,\theta)\right) = \order\left(c(\calX,\theta)\gtt_t\log t\right). 
    \end{align}
    
    Combining \pref{eqn: sto-p2} and \pref{eqn: sto-p3} and following a similar analysis in the previous case, we have
    \begin{align*}
        \sum_{s=t_0+1}^t\left(\ell_{s,x_s}-\ell_{s,\hatx}\right) &\leq C + \order\left(d\gtt_{M'}\sqrt{M'}+c(\calX;\theta)\htt\log t\right) + 4\sqrt{\sum_{s=t_0+1}^t\E_0[1-\tildep_{s,\hatx}]\log(T/\delta)}+2\log(T/\delta) \\
        &\leq \order\left(c(\calX, \theta)\log T\log(T|\calX|/\delta) + \sqrt{f_TC_1t_0}+d\gtt_{M'}\sqrt{M'}\right).
    \end{align*}

\end{proof}

Now we are ready to show that once $L$ grows large enough, Phase $2$ never ends. 

\begin{lemma}\label{lem: phase-2-no-end-v2}
If $C\leq \frac{1}{30}\sqrt{f_TC_1L}$, then with probability at least $1-15\delta$, Phase 2 never ends. 
\end{lemma}

\begin{proof}
It suffices to verify the two termination conditions \pref{eqn: alg_condition_1} and \pref{eqn: alg_condition_2} are never satisfied.  \pref{eqn: alg_condition_2} does not hold because of \pref{lem: phase-2-cond-2-v2}. Consider \pref{eqn: alg_condition_1}. Let $t$ be in Phase $2$ and $x\ne \hatx$. According to \pref{eqn: sub-opt-x-est} and \pref{eqn: opt-x-est}, we have with probability $1-5\delta$,
    \begin{align*}
        \left|\sum_{s=1}^{t_0}(\ell_{s,x}-\ellhat_{s,x})+\sum_{s=t_0+1}^t\left(\ell_{s,x}-\Rob_{t,x}\right)\right| &\leq \frac{1.7t\hatDelta_x}{10}, \\
        \left|\sum_{s=1}^{t}(\ell_{s,\hatx}-\ellhat_{s,\hatx})\right| &\leq 3\sqrt{f_TC_1t} \leq 0.15t\hatDelta_x.
    \end{align*}
    Therefore, we have
    \begin{align*}
        \left|t\hatDelta_{t,x}-t\Delta_x\right|&\leq \left|\sum_{s=1}^{t_0}(\ell_{s,x}-\ellhat_{s,x})+\sum_{s=t_0+1}^t\left(\ell_{s,x}-\Rob_{t,x}\right)\right| 
        + \left|\sum_{s=1}^{t}(\ell_{s,\hatx}-\ellhat_{s,\hatx})\right| +C \\ 
        &\leq 0.32t\hatDelta_x +C \leq 0.372t\hatDelta_x. \tag{$C\leq \frac{1}{30}\sqrt{f_TC_1t}\leq 0.052t\hatDelta_{\min}$}
    \end{align*}
    This means that
    \begin{align*}
        t\hatDelta_{t,x} \leq t\Delta_x + 0.372t\hatDelta_x \leq \frac{1}{0.7}t\hatDelta_x + 0.372 t\hatDelta_x \leq 1.81t\hatDelta_x, \\
        t\hatDelta_{t,x} \geq t\Delta_x - 0.372t\hatDelta_x \geq \frac{1}{1.3}t\hatDelta_x - 0.372 t\hatDelta_x \geq 0.39t\hatDelta_x.
    \end{align*}
    Therefore, \pref{eqn: alg_condition_1} is not satisfied.
\end{proof}

Finally, we prove the regret bound for the corrupted stochastic setting.

\begin{proof}[\textbf{Proof of \pref{thm: sto-and-corrupt}}]
        First, we consider the pure stochastic setting with $C=0$. According to \pref{lem: phase-1-v1}, we know that the algorithm has only one epoch as $C\leq \frac{1}{30}\sqrt{f_TC_1L}$ is satisfied in the first epoch. Specifically, after at most $\frac{900f_TC_1}{\Delta_{\min}^2}$ rounds in Phase $1$, the algorithm goes to Phase $2$ and never goes back to Phase $1$. Then we can directly apply the second claim in \pref{lem: phase-2-cond-2-v2} to get the regret bound in the stochastic setting. Specifically, we bound the regret in Phase $1$ by $\order\left(\sqrt{C_1L_0}+\sqrt{C_1\cdot \frac{900\fT C_1}{\Delta_{\min}^2}}\right) = \order\left(\sqrt{C_1L_0}+\frac{C_1\sqrt{\log T}}{\Delta_{\min}}\right)$. For the regret in Phase $2$, according to the second claim in \pref{lem: phase-2-cond-2-v2}, we bound the regret by $\order\left(c(\calX;\theta)\log T\log\frac{T|\calX|}{\delta}+d\gtt_{M'}\sqrt{M'}\right) = \order\left(c(\calX;\theta)\log T\log\frac{T|\calX|}{\delta}+M^*\log^{\frac{3}{2}}\frac{1}{\delta}\right)$, where $M^*$ is the same as the one in \pref{eqn: M-star}. Combining them together proves the first claim.

        Now we consider the corrupted stochastic setting with $C>0$. Suppose that we are in the epoch with $L=L^*$, which is the first epoch such that $L^*\geq\max\left\{\frac{900f_TC_1}{\Delta_{\min}^2}, \frac{900C^2}{f_TC_1}\right\}$. Therefore, in previous epochs, we have $L\leq\max\left\{\frac{900f_TC_1}{\Delta_{\min}^2}, \frac{900C^2}{f_TC_1}\right\}$. According to \pref{lem: phase-1-v1}, we have $t_0\leq\max\left\{\frac{900f_TC_1}{\Delta_{\min}^2}, \frac{900C^2}{f_TC_1}\right\}$ in the previous epoch and $L^*=2t_0$.
        
        We bound the regret before this epoch, as well as the regret in the first phase of this epoch by the adversarial regret bound: 
        \begin{align*}
        \order\left(\sqrt{C_1 f_T L^*} + \sqrt{C_1L_0}\right)
        = \order\left(\sqrt{C_1L_0}+\sqrt{C_1f_T} \times \left(\frac{\sqrt{f_TC_1}}{\Delta_{\min}} + \frac{C}{\sqrt{f_TC_1}}\right)\right) = \order\left(\sqrt{C_1L_0}+\frac{f_TC_1}{\Delta_{\min}} +C\right). 
        \end{align*}
        If we use \ghp as the adversarial linear bandit algorithm and $f_T=\log T$, then the above is upper bounded by 
        \begin{align*}
            \order\left(\frac{d\log T\log(T|\calX|/\delta)}{\Delta_{\min}}+C\right). 
        \end{align*}

        For Phase $2$ of the epoch with $L=L^*$, according to \pref{lem: phase-1-v1}, we know that this phase will never end and by definition of $L^*$, we have $C\leq\frac{1}{30}\sqrt{f_TC_1L}$. Note that in this phase $\hatDelta_x\in [0.7\Delta_{x}, 1.3\Delta_x]$. 
        
        Therefore, by taking a summation over $t$ on \pref{eqn: p2-regret}, we bound the regret in this interval by
        \begin{align*}
            \order\left(\frac{d\gtt_T \log T}{\Delta_{\min}}\right) = \order\left(\frac{d\log T\log(T|\calX|/\delta)}{\Delta_{\min}}\right). 
        \end{align*}
        Combining the regret bounds finishes the proof of the second claim.
    \end{proof}

\newcommand{\PP}{\mathbb{P}}
\newcommand{\hatN}{\widehat{N}}
\newcommand{\KL}{\text{\rm KL}}
\newcommand{\reg}{\text{reg}}
\section{Lower Bound}\label{app:lower_bound}
In this section, we prove that the $\log^2 T$ factor in our bound for the stochastic setting is unavoidable if the same algorithm also achieves sublinear regret with high probability in the adversarial case.
The full statement is in \pref{thm: lower bound}, and we first present some definitions and related lemmas.
We fix a stochastic linear bandit instance (i.e., we fix the parameter $\theta$ and the action set $\calX$). Assume that $\|x\|\leq 1$  for all $x\in\calX$ and $\|\theta\|\leq \frac{1}{4}$. 
We call this instance the \emph{first environment}. The observation $y_t$ is generated according to the following Bernoulli distribution\footnote{ In the Bernoulli noise case, we are only looking at a subclass of problems (i.e., those with $\|\theta\|\leq\frac{1}{4}$ and $\|x\|\leq 1$). For problems that are outside this class, the Bernoulli noise case might be much easier than the Gaussian noise case.}: 
\begin{align*}
    y_t = \begin{cases}
        1 &\text{\ with probability\ } \frac{1}{2}+\frac{1}{2}\inner{x_t,\theta}, \\
        -1 & \text{\ with probability\ } \frac{1}{2}-\frac{1}{2}\inner{x_t,\theta}.
    \end{cases}
\end{align*}

Again, let $c(\calX, \theta)$ be the solution of the following optimization problem:
\begin{equation}\label{eqn: op-lower-bound}
    \begin{aligned}
    &\inf_{N\in[0, \infty)^\calX} \sum_{x\in \calX\backslash\{x^*\}}N_x\Delta_x \\
    s.t. \qquad & \|x\|^2_{H(N)^{-1}}\leq \frac{\Delta_x^2}{2}, \forall x\in \calX^-=\calX\backslash\{x^*\},
    \end{aligned}
\end{equation}
where $H(N)=\sum_{x\in\calX}N_x xx^\top $. We also define $\Delta_{\min}=\min_{x\neq x^*}\Delta_x$. 

For a fixed $\gamma\in(0,1)$ (which is chosen later), we divide the whole horizon into intervals of length 
\begin{align*}
    T^\gamma, \frac{4}{\Delta_{\min}}T^{\gamma}, \left(\frac{4}{\Delta_{\min}}\right)^2 T^{\gamma}, \ldots, \left(\frac{4}{\Delta_{\min}}\right)^{S-1}T^{\gamma},
\end{align*}
where $S=\Theta\left(\frac{1-\gamma}{\log\frac{4}{\Delta_{\min}}}\log T\right)$. We denote these intervals as $\calI_1, \ldots, \calI_S$. Observe that $|\calI_i|\geq \frac{3}{\Delta_{\min}}\sum_{j<i}|\calI_j|$ for all $i$. 

\begin{definition}
    Let $U=c_{\reg}S(\log T)^{1-\beta}$, $V=\frac{U}{S}=c_{\reg}(\log T)^{1-\beta}$ for some $\beta\geq 0$ and some universal constant $c_{\reg}$. 
\end{definition}

\begin{assumption}
    Let $\calA$ be a linear bandit algorithm with the following regret guarantee: there is a problem-dependent constant $T_0$ (i.e., depending on $\theta$ and $\calX$) such that for any $T\geq T_0$, 
    \begin{align*}
        \E\left[\Reg(T)\right] = \E\left[\sum_{t=1}^T \one[x_t=x]\Delta_x \right] \leq U\cdot c(\calX, \theta)
    \end{align*}
    for some $\beta\geq 0$.   
\end{assumption}

\begin{definition}
    Let $G_i\triangleq \E\left[\sum_{t\in\calI_i} x_tx_t^\top\right]$ where the expectation $\E$ is with respect to the environment of $\theta$ and the algorithm $\calA$. Let $G=\sum_{i=1}^S G_i$. 
\end{definition}

\begin{definition}
    Let $T_i(x) \triangleq \sum_{t\in \calI_i} \one[x_t=x]$. 
\end{definition}

\begin{lemma}
     Let $T\geq T_0$. There exists an action $x\neq x^*$ such that 
    \begin{align*} 
        \|x\|^2_{G^{-1}} \geq   \frac{\Delta_x^2}{2U}. 
    \end{align*}
\end{lemma}

\begin{proof}
We use contradiction to prove this lemma. Suppose that for all $x\neq x^*$ we have $\|x\|^2_{G^{-1}} \leq \frac{\Delta_x^2}{2U}$. Then observe that $\|x\|_{\overline{G}^{-1}}^2\leq \frac{\Delta_x^2}{2}$ where $\overline{G}=\frac{1}{U}\cdot G$. Therefore, 
    \begin{align*}
        \hatN_x = \E\left[\sum_{t=1}^T \frac{\one[x_t=x]}{U} \right] 
    \end{align*}
    satisfies the constraint of the optimization problem \pref{eqn: op-lower-bound}. Therefore, 
    \begin{align*}
        c(\calX, \theta)\leq \sum_{x\neq x^*} \hatN_x\Delta_x = \E\left[\sum_{t=1}^T \one[x_t=x]U^{-1}\right] = \frac{\Reg_T}{U}. 
    \end{align*}
    This contradicts with the assumption on the regret bound of $\calA$. 
\end{proof}

\begin{lemma}
    \label{lemma: global to local}
    If there exists an $x\neq x^*$ such that 
    \begin{align*}
        \|x\|^2_{G^{-1}} \geq   \frac{\Delta_x^2}{2U}, 
    \end{align*}
    then there exists $i\in[S]$ such that 
    \begin{align*}
        \|x\|^2_{G_i^{-1}} \geq \frac{\Delta_x^2}{2V}. 
    \end{align*}
\end{lemma}
\begin{proof}
    We use contradiction to prove the lemma. Suppose that $\|x\|^2_{G_i^{-1}}< \frac{\Delta_x^2}{2V}$ for all $i\in[S]$. 
    
    Then 
    \begin{align*}
        \|x\|^2_{G^{-1}}&=\|x\|^2_{\left(G_1+\cdots+G_S\right)^{-1}}\leq \frac{1}{S^2}\left(\|x\|^2_{G_1^{-1}}+\cdots+\|x\|^2_{G_S^{-1}}\right)\\
        &< \frac{1}{S} \frac{\Delta_x^2}{2V} = \frac{\Delta_x^2}{2U}
    \end{align*}
    where in the first inequality we use 
    \begin{align*}
        \left(\frac{1}{S}(G_1+\cdots+G_S)\right)^{-1} \preceq \frac{1}{S}\left(G_1^{-1}+\cdots+G_S^{-1}\right),  
    \end{align*}
    which is a generalization of the ``arithmetic-mean-harmonic-mean inequality'' \citep{mond1996mixed}. 
\end{proof}

\begin{lemma}
    \label{lemma: global to local 2}
    Let $T\geq T_0$. If $\|x\|_{G_i^{-1}}^2 \geq \frac{\Delta_x^2}{2V}$, then $\|x-x^*\|_{G_i^{-1}}^2 \geq \frac{\Delta_x^2}{8V}$. 
\end{lemma}
\begin{proof}
We have
    \begin{align}
        \|x\|_{G_i^{-1}}^2
        \leq 2\|x-x^*\|_{G_i^{-1}}^2 + 2\|x^*\|_{G_i^{-1}}^2
        \leq 2\|x-x^*\|_{G_i^{-1}}^2 + \frac{2\|x^*\|^2}{\E[T_i(x^*)]},  \label{eq: bound x G norm}
    \end{align}
    where the last inequality is because of the definition of $G_i$. For large enough $T$, we must have $\E\left[T_i(x^*)\right]\geq \frac{8\|x^*\|^2}{\Delta_{\min}^2}V$. Otherwise, by Markov's inequality, with probability at least $\frac{1}{2}$, in interval $i$ the algorithm draws $x^*$ at most $\frac{16\|x^*\|^2}{\Delta_{\min}^2}V$ times, and thus the regret of $\calA$ would be at least $\Delta_{\min}\left(|\calI_i| - \frac{16\|x^*\|^2}{\Delta_{\min}^2}V\right)=\Omega\left(\left(\frac{4}{\Delta_{\min}}\right)^{i-1}T^{\gamma}\Delta_{\min} - \frac{16\|x^*\|^2}{\Delta_{\min}}c_{\reg}(\log T)^{1-\beta}\right) = \Omega\left(T^{\gamma}\Delta_{\min}\right)$, violating the assumption on $\calA$. Therefore, for large enough $T$, we have 
    \begin{align*}
        \frac{2\|x^*\|^2}{\E[T_i(x^*)]} \leq \frac{\Delta_{\min}^2}{4V} \leq \frac{1}{2}\|x\|_{G_i^{-1}}^2,
    \end{align*}
    where the last inequality is by our assumption. Combining this with \pref{eq: bound x G norm}, we get 
    \begin{align*}
        \|x\|_{G_i^{-1}}^2 \leq 4\|x-x^*\|^2_{G_i^{-1}}
    \end{align*}
    and the conclusion follows based on our assumption. 
\end{proof}

Note that when the conclusion of \pref{lemma: global to local 2} holds, that is, $\|x-x^*\|_{G_i^{-1}}^2 \geq \Omega\left(\frac{\Delta_x^2}{V} \right) = \Omega\left(\frac{\Delta_x^2}{\log T}(\log T)^\beta\right)$, it means that the exploration in interval $i$ is not enough, since by the lower bound in \citep{lattimore2017end}, the amount of exploration should make $\|x-x^*\|^2_{G_i^{-1}}\leq O\left(\frac{\Delta_x^2}{\log T}\right)$. Therefore, the next natural idea is to change the parameter $\theta$ in this interval $i$, and argue that the amount of exploration $\calA$ is not enough to ``detect this change with high probability''. 

We now let $i$ be the first interval such that there exists $x$ with $\|x-x^*\|^2_{G_i^{-1}}\geq \frac{\Delta_x^2}{8V}$. Also, we use $x'$ to denote the $x$ that satisfies this condition. Define 
\begin{align*}
    \theta' = \theta - \frac{G_i^{-1}(x'-x^*)}{\|x'-x^*\|^2_{G_i^{-1}}}2\Delta_{x'}. 
\end{align*}
Notice that in this case,
\begin{align*}
    \inner{x'-x^*,\theta'} = \inner{x'-x^*,\theta} - 2\Delta_{x'} = -\Delta_{x'}. 
\end{align*}
That is, $x'$ is a better action than $x^*$ under the parameter $\theta'$. 

We now define the \emph{second environment} as follows: in intervals $1,\ldots, i-1$, the losses are generated according to $\theta$, but in intervals $i,\ldots, S$, the losses are generated according to $\theta'$. We use $\E$ and $\E'$ to denote the expectation under the first environment and the second environment, and $\PP$, $\PP'$ to denote the probability measures respectively. For now, we only focus on interval $i$ (so the probability measure is only over the sequence $(x_t, \ell_{t,x_t})_{t\in \calI_i}$). 

\begin{lemma}
    \label{lemma: KL lower bound}
        $\KL\left(\PP, \PP'\right)\leq 64V$.
\end{lemma}
\begin{proof}
    Note that for any $x$, 
    \begin{align*}
        |\inner{x,\theta'} - \inner{x,\theta}|
        &=  \left|\frac{x^\top G_i^{-1}(x'-x^*)}{\|x'-x^*\|^2_{G_i^{-1}}}2\Delta_{x'}\right|\\
        &= \left|\frac{2x^\top G_i^{-1}(x'-x^*)(x'-x^*)^\top \theta}{\|x'-x^*\|^2_{G_i^{-1}}} \right| \\
        &\leq \frac{2\|x\|\|\Phi\|_{\text{op}}\|\theta\|}{\text{tr}(\Phi)}  \tag{let $\Phi=G_i^{-1}(x'-x^*)(x'-x^*)^\top $}  \\
        &\leq 2\|x\|\|\theta\| \leq \frac{1}{2}  \tag{by the assumption $\|\theta\|\leq \frac{1}{4}$}.
    \end{align*}
    Therefore, $|\inner{x,\theta'}|\leq |\inner{x,\theta}|+\frac{1}{2}\leq \frac{3}{4}$. 
    
    Notice that for $p,q\in[-\frac{3}{4}, \frac{3}{4}]$, the KL divergence between the following two distributions: 
    \begin{align*}
        y = \begin{cases}
            1 &\text{\ with probability\ }\frac{1}{2}+\frac{1}{2}p \\
            -1 &\text{\ with probability\ }\frac{1}{2}-\frac{1}{2}p
        \end{cases}
        \qquad \text{and} \qquad 
        y = \begin{cases}
            1 &\text{\ with probability\ }\frac{1}{2}+\frac{1}{2}q \\
            -1 &\text{\ with probability\ }\frac{1}{2}-\frac{1}{2}q
        \end{cases}
    \end{align*}
    is 
    \begin{align}\label{eqn: bound-KL}
        \text{kl}(p,q) \triangleq \frac{1}{2}(1+p) \ln \frac{1+p}{1+q} + \frac{1}{2}(1-p) \ln \frac{1-p}{1-q} \leq 2(p-q)^2. 
    \end{align}
    Therefore,
    \begin{align*}
        \KL\left(\PP, \PP'\right) &\leq \sum_{x\in \calX} \E[T_i(x)]\text{kl}\left(\inner{x,\theta}, \inner{x,\theta'}\right) \tag{by Lemma 1 in \citep{gerchinovitz2016refined}}\\
        &\leq 2\sum_{x} \E[T_i(x)]\inner{x,\theta-\theta'}^2 \tag{by \pref{eqn: bound-KL}}\\
        &= 2\|\theta-\theta'\|_{G_i}^2 \tag{by definition of $G_i$}\\
        &= \frac{8\Delta_{x'}^2}{\|x'-x^*\|_{G_i^{-1}}^2} \tag{by definition of $\theta'$}\\
        &\leq 64V \tag{by the choice of $x'$}. 
    \end{align*}
\end{proof}

Finally, we are ready to present the lower bound.
Roughly speaking, it shows that if an algorithm achieves $\order(c(\calX,\theta)\log^z T)$ regret in the stochastic case with $z \in [1,2)$, then it cannot be robust in the adversarial setting, in the sense that it cannot guarantee a regret bound such as $o(T) \cdot  \text{\rm poly} (d, \ln(1/\delta))$ with probability at least $1-\delta$.

\begin{theorem}\label{thm: lower bound}
    For any $\gamma \in(0,1)$, if an algorithm guarantees a pseudo regret bound of  
    \begin{align*}
        c_{\reg}\cdot c(\calX,\theta)\frac{1-\gamma}{\log \frac{4}{\Delta_{\min}}}(\log T)^2   
    \end{align*} for constant $c_{\reg}=\frac{\gamma}{256}$ 
    in stochastic environments for all sufficiently large $T$, then there exists an adversarial environment such that with probability at least $\frac{1}{4} T^{-\frac{1}{4}\gamma}$, the regret of the same algorithm is at least $\frac{1}{6}T^{\gamma}\Delta_{\min}$.
\end{theorem}

\begin{proof}
    Let $A$ be the event: $\left\{T_i(x^*)\leq \frac{|\calI_i|}{2}\right\}$. By Lemma 5 in \citep{lattimore2017end} and choosing $\beta = 0$, we have
    \begin{align}
        \PP(A) + \PP'(A^c) &\geq \frac{1}{2}\exp(-\KL(\PP, \PP'))   \nonumber \\
        &\geq \frac{1}{2}\exp(-64V) \tag{by \pref{lemma: KL lower bound}}\\
        &= \frac{1}{2}\exp\left(-64c_{\reg}\log T\right)  \tag{by definition of $V$}\nonumber \\
        &= \frac{1}{2}\left(\frac{1}{T}\right)^{64c_{\reg}}.   \label{eq: combined event}
    \end{align}
    
    Notice that when event $A$ happens under the first environment, the regret is at least $\frac{|\calI_i|\Delta_{\min}}{2}$. By the assumption on $\calA$, we have 
    \begin{align*}
        \PP(A)\times \frac{|\calI_i|\Delta_{\min}}{2} \leq \Reg_T \leq c_{\reg}\cdot c(\calX,\theta)\frac{1-\gamma}{\log \frac{4}{\Delta_{\min}}}(\log T)^2, 
    \end{align*}
    implying that 
    \begin{align*}
        \PP(A)\leq \frac{c_{\reg}\cdot c(\calX,\theta)\frac{1-\gamma}{\log \frac{4}{\Delta_{\min}}}(\log T)^2}{\left(\frac{4}{\Delta_{\min}}\right)^{i-1}T^{\gamma}\Delta_{\min}} \leq \frac{c_{\reg}\cdot c(\calX,\theta)(\log T)^2}{T^{\gamma}}.
    \end{align*}
    
    Combining this with \eqref{eq: combined event}, we get 
    \begin{align}
        \PP'(A^c) \geq 0.5\cdot T^{-64c_{\reg}} - c_{\reg}\cdot c(\calX,\theta)(\log T)^2\cdot T^{-\gamma}.
    \end{align}
    Choose $c_{\reg}=\frac{\gamma}{256}$, we have $\PP'(A^c)\geq 0.25\cdot T^{-64c_{\reg}}$ for large enough $T$.
    Then notice that when $A^c$ happens under the second environment, the regret within interval $\calI_i$ (against comparator $x'$) is at least $\frac{|\calI_i|\Delta_{\min}}{2}$. Since under the second environment, the learner may have negative regret against $x'$ in interval $1, \ldots, i-1$, in the best case the regret against $x'$ in interval $1,\ldots, i$ is at least 
    \begin{align*}
        \frac{|\calI_i|\Delta_{\min}}{2} - \left(|\calI_1|+\cdots+|\calI_{i-1}|\right)\geq \frac{|\calI_i|\Delta_{\min}}{6}\geq \frac{T^{\gamma} \Delta_{\min}}{6}. 
    \end{align*}
    
    In conclusion, in the second environment, algorithm $\calA$ suffers at least $\frac{T^{\gamma}\Delta_{\min}}{6}$ regret in the first $i$ intervals with probability at least 
$
        0.25\cdot T^{-\frac{1}{4}\gamma} 
$. 
\end{proof}


\section{Adversarial Linear Bandit Algorithms with High-probability Guarantees}\label{app: example algorithm}
In this section, we show that the algorithms of \citep{bartlett2008high} and \citep{lee2020bias} both satisfy \pref{assum:  adversarial alg}.

\subsection{\ghp}
\begin{algorithm}\caption{\ghp}\label{alg: ghp}
    \textbf{Input}: $\calX,\gamma,\eta,\delta'$, and John's exploration distribution $q \in \calP_\calX$.\\
    Set $\forall x \in\calX$, $w_1(x)=1$, and $W_1=|\calX|$.\\
    \For{$t=1$ to $T$}{
        Set $p_t(x)=(1-\gamma)\frac{w_t(x)}{W_t}+\gamma q(x),~\forall x\in\calX$.\\
        Sample $x_t$ according to distribution $p_t$.\\
        Observe loss $y_t=\ell_{t,x_t}+\epsilon_t(x_t)$, where $\ell_{t,x_t}=\inner{x_t,\ell_t}$.\\
        Compute $S(p_t)=\sum_{x\in\calX}p_t(x)xx^\top$ and $\hatell_t=S(p_t)^{-1}x_t\cdot y_t$.\\
        $\forall x\in\calX$, compute
        \[
        \hatell_{t,x}=\inner{x,\hatell_t},\qquad \widetilde{\ell}_{t,x}=\hatell_{t,x}-2\|x\|_{S(p_t)^{-1}}^2\sqrt{\frac{\log(1/\delta')}{dT}},\qquad w_{t+1}=w_t(x)\exp\left(-\eta\widetilde{\ell}_{t,x}\right).
        \]
        Compute $W_{t+1}=\sum_{x\in\calX}w_{t+1}(x)$.
    }
\end{algorithm}


We first show \ghp in \pref{alg: ghp} for completeness.
We remark the differences between the original version and one shown here. 
First, we consider the noisy feedback $y_t$ instead of the zero-noise feedback $\ell_{t,x_t}$.
However, most analysis in \citep{bartlett2008high} still holds.
Second, instead of using the barycentric spanner exploration (known to be suboptimal),
we use John's exploration shown to be optimal in \citet{bubeck2012towards}.
With this replacement, Lemma 3 in \citet{bartlett2008high} can be improved to $|\hatell_{t,x}|\le d/\gamma$ and $\|x\|_{S(p_t)^{-1}}^2\le d/\gamma$.

Now, consider martingale difference sequence $M_t(x)=\hatell_{t,x}-\ell_{t,x}$. 
We have $|M_t(x)|\le \frac{d}{\gamma}+1\triangleq b$ and 
$$
\sigma=\sqrt{\sum^T_{t=1}\text{Var}_t(M_t)}\le\sqrt{\sum_{t=1}^{T}  \|x\|^2_{S(p_t)^{-1}}}.
$$
Using Lemma 2 in \citep{bartlett2008high}, we have that with probability at least $1-2\delta' \log_2T$ (set $\delta'=\delta/(|\calX|\log_2(T))$), 
\begin{align}
\left|\sum_{t=1}^T(\ellhat_{t,x} - \ell_{t,x}) \right|&\le 2\max\left\{2\sigma,b\sqrt{\log(1/\delta')}\right\}\sqrt{\log(1/\delta')}\nonumber\\
&\le4\sigma\sqrt{\log(1/\delta')} + 2b\cdot{\log(1/\delta')}\nonumber\\
&\le4\sqrt{\sum_{t=1}^{T}  \|x\|^2_{S(p_t)^{-1}}}\sqrt{\log(1/\delta')} +2\left(\frac{d}{\gamma}+1\right){\log(1/\delta')}\nonumber\\
&\le\underbrace{\frac{1}{C_2}\left(\sum_{t=1}^{T} \|x\|^2_{S(p_t)^{-1}}\sqrt{\frac{\log(1/\delta')}{dT}}\right)+4C_2\sqrt{dT\log(1/\delta')}+2\left(\frac{d}{\gamma}+1\right){\log(1/\delta')}}_{\triangleq \dev_{T,x}}, \label{eq: devtx}
\end{align}
where the last inequality is by AM-GM inequality. Note that $\hatell_{t,x}=\widetilde{\ell}_{t,x}+2\|x\|^2_{S(p_t)^{-1}}\sqrt{\frac{\log(1/\delta')}{dT}}$. 
Plugging this into \pref{eq: devtx}, we have with probability at least $1-2\delta' \log_2T$,
\begin{align}\label{eq: ghp lemma5}
    \sum_{t=1}^T\widetilde{\ell}_{t,x}\le\sum_{t=1}^T\ell_{t,x}-{\left(\sum_{t=1}^{T} \|x\|^2_{S(p_t)^{-1}}\sqrt{\frac{\log(1/\delta')}{dT}}\right)+4C_2\sqrt{dT\log(1/\delta')}+2\left(\frac{d}{\gamma}+1\right){\log(1/\delta')}}.
\end{align}
The counterpart of Lemma 6 in \citet{bartlett2008high} shows that with probability at least $1-\delta$,
\begin{align}
    \sum^T_{t=1}\ell_{t,x_t}-\sum^T_{t=1}\sum_{x\in\calX}p_t(x)\hatell_{t,x}\le (\sqrt{d}+1)\sqrt{2T\log(1/\delta)}+\frac{4}{3}\log(1/\delta)\left(\frac{d}{\gamma}+1\right).\label{eq: ghp lemma6}
\end{align}
Using \pref{eq: ghp lemma5}, we have the counterpart of Lemma 7 in \citet{bartlett2008high} as follows: with probability at least $1-2\delta$,
\begin{align}
    \gamma\sum_{t=1}^T\sum_{x\in \calX}q(x)\widetilde{\ell}_{t,x}&\le\gamma\sum_{t=1}^T\sum_{x\in \calX}q(x)\ell_{t,x}+{4C_2\gamma\sqrt{dT\log(1/\delta')}+2\gamma\left(\frac{d}{\gamma}+1\right){\log(1/\delta')}}\nonumber\\
    &\le \gamma T+{4C_2\gamma\sqrt{dT\log(1/\delta')}+2\left({d}+\gamma\right){\log(1/\delta')}}. \label{eq: ghp lemma7}
\end{align}
The counterpart of Lemma 8 in \citet{bartlett2008high} is: with probability at least $1-\delta$,
\begin{align}
    \sum^T_{t=1}\sum_{x\in\calX}p_t(x){\hatell_{t,x}}^2\le dT+\frac{d}{\gamma} \sqrt{2T\log(1/\delta)}.\label{eq: ghp lemma8}
\end{align}
Plugging \pref{eq: ghp lemma6}, \pref{eq: ghp lemma7}, and \pref{eq: ghp lemma8}, into Equation (2) in \citep{bartlett2008high}, with have with probability at least $1-3\delta$,
\begin{align}
    \log\frac{W_{T+1}}{W_1}&\le\frac{\eta}{1-\gamma}\left(-\sum^T_{t=1}\ell_{t,x_t}+2\sqrt{dT\log(1/\delta')}+ (\sqrt{d}+1)\sqrt{2T\log(1/\delta)}+\frac{4}{3}\log(1/\delta)\left(\frac{d}{\gamma}+1\right)+\gamma T+\right.\nonumber\\
    &\quad\left.{4C_2\gamma\sqrt{dT\log(1/\delta')}+2\left({d}+\gamma\right){\log(1/\delta')}}+ 2\eta dT+\frac{2\eta d}{\gamma} \sqrt{2T\log(1/\delta)}+8\eta\log(1/\delta')\sqrt{dT}\right).\label{eq: ghp eq 3}
\end{align}
Again using \pref{eq: ghp lemma5} and Equation (4) in \citep{bartlett2008high}, we have with probability at least $1-\delta$, for all $x\in\calX$,
\begin{align*}
    \log\frac{W_{T+1}}{W_1}&\ge-\eta\left(\sum^T_{t=1}\widetilde{\ell}_{t,x}\right)-\log|\calX|\\
    &\geq -\eta\sum_{t=1}^T\ell_{t,x}+{\eta\left(\sum_{t=1}^{T} \|x\|^2_{S(p_t)^{-1}}\sqrt{\frac{\log(1/\delta')}{dT}}\right)-4\eta C_2\sqrt{dT\log(1/\delta')}-2\eta\left(\frac{d}{\gamma}+1\right){\log(1/\delta')}}-\log|\calX|.
\end{align*}
Combining this with \pref{eq: ghp eq 3} and assuming $\gamma\le\frac{1}{2}$, we have that with probability at least $1-5\delta$, for every $x\in\calX$,
\begin{align}
    \sum^T_{t=1}\ell_{t,x_t}&\le\sum^T_{t=1}\ell_{t,x}-{\frac{1}{2}\left(\sum_{t=1}^{T} \|x\|^2_{S(p_t)^{-1}}\sqrt{\frac{\log(1/\delta')}{dT}}\right)+4C_2\sqrt{dT\log(1/\delta')}+2\left(\frac{d}{\gamma}+1\right){\log(1/\delta')}}+\frac{\log|\calX|}{\eta}+\nonumber\\
    &\qquad 2\sqrt{dT\log(1/\delta')}+(\sqrt{d}+1)\sqrt{2T\log(1/\delta)}+\frac{4}{3}\log(1/\delta)\left(\frac{d}{\gamma}+1\right)+\gamma T+2C_2\sqrt{dT\log(1/\delta')}+\nonumber\\
    &\qquad 2\left({d}+\gamma\right){\log(1/\delta')}+2\eta dT+\frac{2\eta d}{\gamma} \sqrt{2T\log(1/\delta)}+8\eta\log(1/\delta')\sqrt{dT}.\nonumber
\end{align}
Recalling the definition of $\dev_{T,x}$ in \pref{eq: devtx} and combining terms, we have
\begin{align}
    \sum^T_{t=1}\ell_{t,x_t}&\le\sum^T_{t=1}\ell_{t,x}-C_2\cdot\dev_{T,x}+\order\left(\sqrt{dT\log(1/\delta')}+\frac{d}{\gamma}\log(1/\delta')\right)+\frac{\log|\calX|}{\eta}+\gamma T+2\eta dT+\nonumber\\
    &\quad \frac{2\eta d}{\gamma} \sqrt{2T\log(1/\delta)}+8\eta\log(1/\delta')\sqrt{dT}.
\end{align}
It remains to decide $\eta$ and $\gamma$.
Note that the analysis of~\citep{bartlett2008high} requires $|\eta\widetilde{\ell}_{t,x}|\le 1$.
From the proof of Lemma 4 in \citep{bartlett2008high}, we know that $|\eta\widetilde{\ell}_{t,x}|\le \frac{\eta d}{\gamma}\left(1+2\sqrt{\frac{\log(1/\delta')}{dT}}\right)$.
Thus, we set $\eta=\gamma/\left(d+2d\sqrt{\frac{\log(1/\delta')}{dT}}\right)$ so that $|\eta\widetilde{\ell}_{t,x}|\le 1$ always holds.
Therefore,
\begin{align*}
    \sum^T_{t=1}\ell_{t,x_t}&\le\sum^T_{t=1}\ell_{t,x}-C_2\cdot\dev_{T,x}+\order\left(\sqrt{dT\log(1/\delta')}+\frac{d}{\gamma}\log(|\calX|/\delta')\right)+\frac{2}{\gamma}\sqrt{\frac{d\log^3(|\calX|/\delta')}{T}}+3\gamma T+8\gamma\log(1/\delta')\sqrt{\frac{T}{d}}.
\end{align*}
Choosing $\gamma=\min\left\{\frac{1}{2},\sqrt{\frac{d\log({|\calX|}/{\delta'})}{T}}\right\}$, we have with probability at least $1-7\delta$, for all $x\in \calX$,
\begin{align*}
    \sum^T_{t=1}\ell_{t,x_t}&\le\sum^T_{t=1}\ell_{t,x}-C_2\cdot\dev_{T,x}+\order\left(\sqrt{dT\log(|\calX|/\delta')}+d\log^{\tfrac{3}{2}}(|\calX|/\delta')\right)\\
    &\le\sum^T_{t=1}\ell_{t,x}-C_2\cdot\dev_{T,x}+\order\left(\sqrt{dT\log(|\calX|\log_2(T)/\delta)}+d\log^{\tfrac{3}{2}}(|\calX|\log_2(T)/\delta)\right)\tag{$\delta'=\delta/(|\calX|\log_2(T))$} \\
    &\leq \sum^T_{t=1}\ell_{t,x}-C_2 \left|\sum_{t=1}^T(\ell_{t,x} - \ellhat_{t,x}) \right|
    +\order\left(\sqrt{dT\log(|\calX|\log_2(T)/\delta)}+d\log^{\tfrac{3}{2}}(|\calX|\log_2(T)/\delta)\right), \tag{\pref{eq: devtx}}
\end{align*}
which proves \pref{eqn: eqL guarantee 1}.

\subsection{The algorithm of \citep{lee2020bias}}
Now we introduce another high-probability adversarial linear bandit algorithm from \citep{lee2020bias}.
The regret bound of this algorithm is slightly worse than \citep{bartlett2008high}.
However, the algorithm is efficient when there are infinite or exponentially many actions.
For the concrete pseudocode of the algorithm, we refer the readers to Algorithm 2 of \citep{lee2020bias}.
Here, we focus on showing that it satisfies \pref{eqn: eqL guarantee 1}.

We first restate Lemma B.15 in \citep{lee2020bias} with explicit logarithmic factors:
Algorithm 2 of~\citep{lee2020bias} with $\eta\le \frac{C_3}{d^2\logdt^3\log(\logdt/\delta)}$ for some universal constant $C_3> 0$ guarantees that with probability at least $1-\delta$,
\begin{align}
    \sum^T_{t=1}\inner{x_t-x,\ell_t}\le\order\left(\frac{d\log T}{\eta}+\eta d^2 T+\logdt^2\sqrt{T\log(\logdt/\delta)}\right)+ \dev_{T,x} \cdot\left(C_4-\frac{1}{C_5\eta d^2\logdt^3\sqrt{T\log(\logdt/\delta)}}\right),\label{eq: lemma B.15}
\end{align}
where $\logdt=\log(dT)$,  $\dev_{T,x}$ is an upper bound on $\left|\sum_{t=1}^T(\ell_{t,x} - \ellhat_{t,x}) \right|$ with probability $1-\delta$, $C_4,C_5 > 0$ are two universal constants, and we replace the self-concordant parameter in their bound by a trivial upper bound $d$.

Therefore, choosing $\eta=\min\left\{\frac{C_3}{d^2\logdt^3\log(\logdt/\delta)},\frac{1}{2C_4C_5d^2\logdt^3\sqrt{T\log(\logdt/\delta)}},\frac{1}{2C_2C_5d^2\logdt^3\sqrt{T\log(\logdt/\delta)}}\right\}$ for some $C_2 \geq 20$, the coefficient of $\dev_{T,x}$ becomes at most $-C_2$, leading to
\begin{align*}
    \sum^T_{t=1}\inner{x_t-x,\ell_t}&\le\order\left(d^3\logdt^4\log(\logdt/\delta)+d^3\logdt^4\sqrt{T\log(\logdt/\delta)}\right)-C_2\cdot\dev_{T,x}\\
    &\le \order\left(d^3\logdt^4\log(\logdt/\delta)+d^3\logdt^4\sqrt{T\log(\logdt/\delta)}\right)-C_2\cdot\left|\sum_{t=1}^T(\ell_{t,x} - \ellhat_{t,x}) \right|.
\end{align*}
Finally, using a union bound over all $x$ similar to Theorem B.16 of \citep{lee2020bias}, we get with probability at least $1-2\delta$, for every $x\in\calX$,
\begin{align*}
    \sum^T_{t=1}\inner{x_t-x,\ell_t}&\le \order\left(d^3\logdt^4\log(\logdt/\delta'')+d^3\logdt^4\sqrt{T\log(\logdt/\delta'')}\right)-C_2\left|\sum_{t=1}^T(\ell_{t,x} - \ellhat_{t,x}) \right|
\end{align*}
where $\delta''=\delta/(|\calX| T)$.
Therefore, we conclude that this algorithm satisfies \pref{eqn: eqL guarantee 1} as well.

\section{Auxiliary Lemmas}
In this section, we provide several auxiliary lemmas that we have used in the analysis.
\begin{lemma}{(Lemma A.2 of \citep{shalev2014understanding})}\label{lem: tech-2}
    Let $a\geq 1$ and $b>0$. If $x\geq 4a\log(2a)+2b$, then we have $x\geq a\log(x)+b$.    
\end{lemma}


\begin{lemma} [Concentration inequality for Catoni's estimator \citep{wei2020taking}]
\label{lem: catoni concentration}
Let $\mathcal{F}_0 \subset \cdots \subset \mathcal{F}_n$ be a filtration, and $X_1, \ldots, X_n$ be real random variables such that $X_i$ is $\mathcal{F}_i$-measurable, $\E[X_i|\mathcal{F}_{i-1}]=\mu_i$ for some fixed $\mu_i$, and $ \sum_{i=1}^n \E[(X_i-\mu_i)^2|\mathcal{F}_{i-1}] \leq V$ for some fixed $V$. 
Denote $\mu\triangleq \frac{1}{n}\sum_{i=1}^n \mu_i$ and
let $\widehat{\mu}_{n,\alpha}$ be the Catoni's robust mean estimator of $X_1, \ldots, X_n$ with a fixed parameter $\alpha > 0$, that is,
$\widehat{\mu}_{n,\alpha}$ is the unique root of the function
\begin{align*}
      f(z) = \sum_{i=1}^n \psi(\alpha(X_i-z))
\end{align*}
where
\[
\psi(y) = \begin{cases}
\ln(1+y+y^2/2), &\text{if $y\geq0$,} \\
-\ln(1-y+y^2/2), &\text{else.}
\end{cases}
\]
Then for any $\delta\in (0,1)$,
as long as $n$ is large enough such that $n \geq \alpha^2(V + \sum_{i=1}^n (\mu_i-\mu)^2) + 2\log(1/\delta)$, 
we have with probability at least $1-2\delta$,
\begin{align*}
     |\widehat{\mu}_{n,\alpha} - \mu| \leq \frac{\alpha (V+\sum_{i=1}^n (\mu_i - \mu)^2)}{n} + \frac{2\log(1/\delta)}{\alpha n}.
\end{align*}
Choosing $\alpha$ optimally, we have
\begin{align*}
    |\widehat{\mu}_{n,\alpha} - \mu| \leq \frac{2}{n}\sqrt{2\left(V+\sum_{i=1}^n(\mu_i-\mu)^2\right)\log(1/\delta)}.
\end{align*}
In particular, if $\mu_1 = \cdots = \mu_n = \mu$, we have
\begin{align*}
     |\widehat{\mu}_{n,\alpha} - \mu| \leq \frac{2}{n}\sqrt{2V\log(1/\delta)}.
\end{align*}
\end{lemma}

\end{document}